\documentclass[lettersize,journal]{IEEEtran}
\usepackage{amsmath,amsfonts}
\usepackage{algorithmic}
\usepackage{algorithm}
\usepackage{array}
\usepackage[caption=false,font=normalsize,labelfont=sf,textfont=sf]{subfig}
\usepackage{textcomp}
\usepackage{stfloats}
\usepackage{url}
\usepackage{verbatim}
\usepackage{graphicx}
\usepackage{cite}
\usepackage{multirow}
\usepackage{booktabs}
\usepackage{adjustbox} 
\usepackage[colorlinks,
linkcolor=blue,
anchorcolor=blue,
citecolor=blue]{hyperref}
\usepackage{xr}
\usepackage{ulem}
\usepackage{amsthm}

\usepackage{enumitem}
\usepackage{pifont} 
\usepackage[table,dvipsnames]{xcolor}
\usepackage{booktabs} 
\usepackage{makecell}

\newtheorem{assumption}{Assumption}
\newtheorem{lemma}{Lemma}
\newtheorem{theorem}{Theorem}

\begin{document}

\title{FedeCouple: Fine-Grained Balancing of Global-Generalization and Local-Adaptability in Federated Learning}

\author{
    Ming Yang,~\IEEEmembership{Member,~IEEE,}
    \and
    Dongrun Li,
    \and
    Xin Wang,~\IEEEmembership{Member,~IEEE,}
    \and
    Feng Li,~\IEEEmembership{Member,~IEEE,}
    \and
    Lisheng Fan,
    \and
    Chunxiao Wang,
    \and
    Xiaoming Wu,
    and
    Peng Cheng,~\IEEEmembership{Member,~IEEE}
    \thanks{M. Yang, D. Li, X. Wang, C. Wang, and X. Wu are with the Key Laboratory of Computing Power Network and Information Security, Ministry of Education, Shandong Computer Science Center, Qilu University of Technology (Shandong Academy of Sciences), Jinan 250014, P. R. China, and also with Shandong Provincial Key Laboratory of Industrial Network and Information System Security, Shandong Fundamental Research Center for Computer Science, Jinan 250014, China. Emails: {\small yangm@sdas.org, 10431230017@stu.qlu.edu.cn, xinwang@qlu.edu.cn, wangchx@sdas.org, wuxm@sdas.org}.}
    \thanks{F. Li is with the School of Computer Science and Technology, Shandong University, Qingdao 266237, China, and also with the Key Laboratory of Computing Power Network and Information Security, Ministry of Education, Qilu University of Technology (Shandong Academy of Sciences), Jinan 250014, China. Email: {\small fli@sdu.edu.cn}.}
    \thanks{L. Fan is with the School of Computer Science and Cyber Engineering, Guangzhou University, Guangzhou 510006, China. Email: {\small lsfan@gzhu.edu.cn}.}
    \thanks{P. Cheng is with the State Key Lab. of Industrial Control Technology, Zhejiang University, Hangzhou 310027, China. Email:
    {\small lunarheart@zju.edu.cn}.}
}

\markboth{}%
{Shell \MakeLowercase{\textit{et al.}}: A Sample Article Using IEEEtran.cls for IEEE Journals}


\maketitle

\definecolor{darkgreen}{rgb}{0.0, 0.5, 0.0}

\begin{abstract}
In privacy-preserving mobile network transmission scenarios with heterogeneous client data, personalized federated learning methods that decouple feature extractors and classifiers have demonstrated notable advantages in enhancing learning capability. However, many existing approaches primarily focus on feature space consistency and classification personalization during local training, often neglecting the local adaptability of the extractor and the global generalization of the classifier. This oversight results in insufficient coordination and weak coupling between the components, ultimately degrading the overall model performance. 
To address this challenge, we propose FedeCouple, a federated learning method that balances global generalization and local adaptability at a fine-grained level. Our approach jointly learns global and local feature representations while employing dynamic knowledge distillation to enhance the generalization of personalized classifiers. 
{We further introduce anchors to refine the feature space; their strict locality and non-transmission inherently preserve privacy and reduce communication overhead.}
{Furthermore, we provide a theoretical analysis proving that FedeCouple converges for nonconvex objectives, with iterates approaching a stationary point as the number of communication rounds increases.} 
Extensive experiments conducted on five image-classification datasets demonstrate that FedeCouple consistently outperforms nine baseline methods in effectiveness, stability, scalability, and security. Notably, in experiments evaluating effectiveness, FedeCouple surpasses the best baseline by a significant margin of $4.3\%$.
\end{abstract}

\begin{IEEEkeywords}
Personalized Federated Learning, Heterogeneous Data, Model Decoupling.
\end{IEEEkeywords}

\section{Introduction}\label{intro}
\IEEEPARstart{N}{owadays}, the ``Artificial Intelligence (AI) $+$ Mobile Network (MN)" paradigm is advancing rapidly, yet it faces significant challenges. 
On the one hand, there is a growing aspiration to transmit diverse data from mobile devices to facilitate various AI model training tasks~\cite{ref3tmc}. 
On the other hand, privacy concerns in MNs deter the public sharing of sensitive data~\cite{ref9tmc,WANG2025}. 
To address this issue, federated learning (FL) has emerged as a promising solution \cite{ref9,ijcai,ecai}. FL is an innovative distributed machine learning approach that enables multiple participants to collaboratively train models without sharing raw data from mobile devices \cite{ref10tmc}. Compared to centralized methods, its distributed architecture allows for the effective aggregation of useful information from each participant while maintaining data privacy \cite{ref12tmc}.

However, traditional FL methods exhibit significantly different performance under heterogeneous and homogeneous data conditions. When data is evenly distributed or independently and identically distributed (IID) across clients, FL methods such as FedAvg \cite{ref9} typically perform well. This is because the training objectives of the clients are aligned, and the data distribution is relatively consistent. As a result, simple parameter averaging can produce an effective global model that generalizes well to both global and local tasks. In contrast, when the data across clients follows non-IID distributions, the performance of these traditional FL algorithms often deteriorates \cite{ref15}. In such scenarios, due to large discrepancies in training tasks across clients, locally trained models tend to converge toward divergent optimization directions. Consequently, simply aggregating the local parameters may fail to reconcile these differences effectively, leading to a global model with suboptimal performance \cite{ref11tmc}.
\begin{figure}[!t]
	\centering
	\includegraphics[width=3.1in]{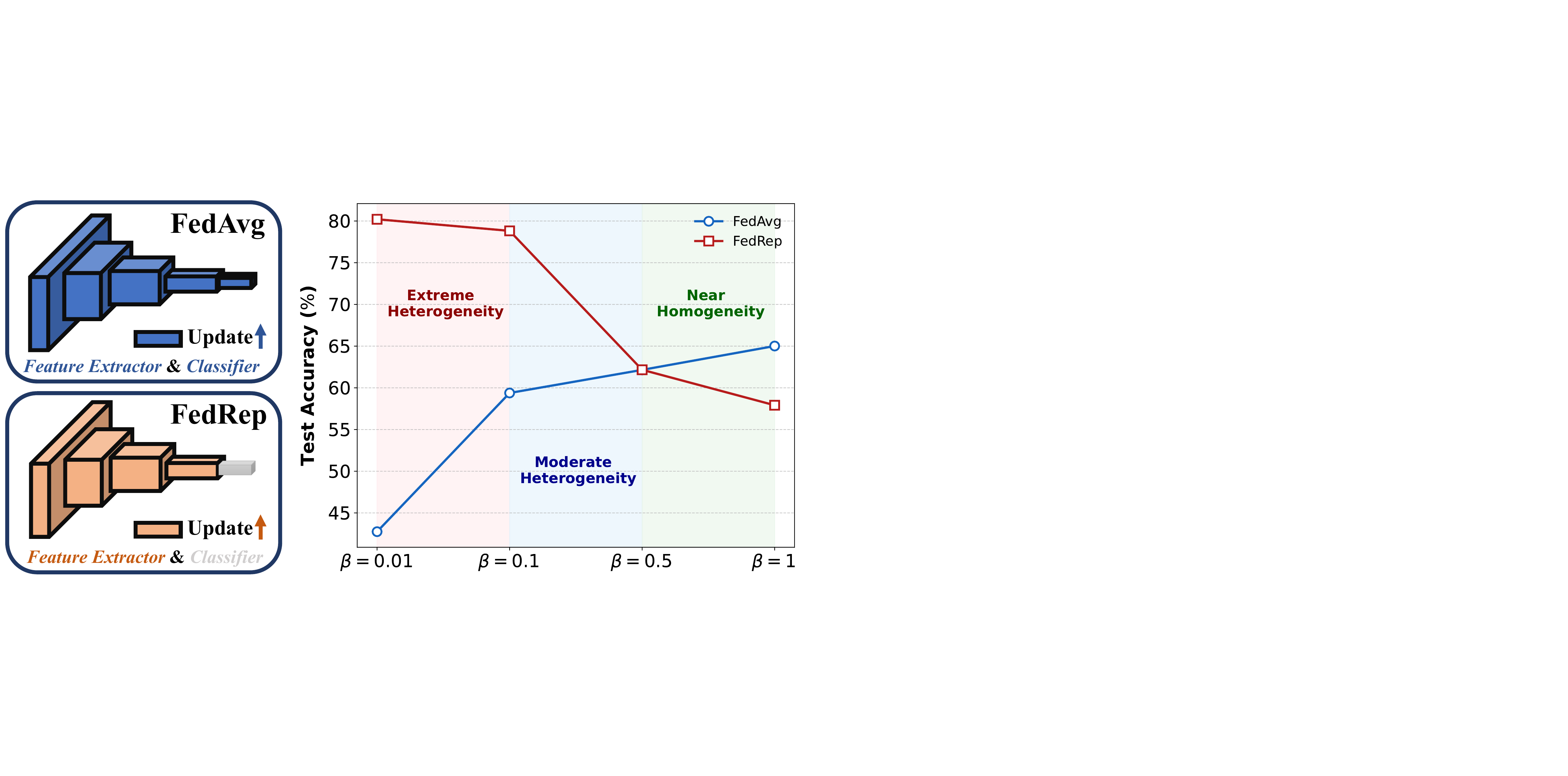}
	\caption{Test accuracy of FedAvg\cite{ref9} and FedRep\cite{ref39} on CIFAR-10\cite{ref53}, where smaller $\beta$ indicates stronger heterogeneity. (Left: framework, right: results)}
	\label{fig0}
    \vspace{-15pt}
\end{figure}

{To address client-level data heterogeneity in FL, the literature pursues two complementary directions. The data-centric approach mitigates class scarcity and non-IID shifts by generating synthetic data. For example, AZSL leverages zero-shot learning with a WGAN-GP to synthesize data for unseen classes without sharing raw data, thereby improving model robustness~\cite{AZSL}. Conversely, the model-centric line focuses on personalized federated learning (PFL), which transfers knowledge across clients while excelling on client-specific tasks.}
By incorporating consensus constraints into local training, PFL enables local models to retain global knowledge while adapting to personalized needs \cite{ref22tmc}. Among PFL methods, decoupling the model into a feature extractor and a classifier, each trained with distinct strategies, has emerged as an effective solution. 
However, the approaches of combining a global feature extractor with a fully personalized classifier struggle to address varying degrees of data heterogeneity effectively. As illustrated in Fig.~\ref{fig0}, these PFL methods outperform traditional approaches only in cases of extreme heterogeneity. When applied to a broader range of heterogeneity levels, their performance is often inferior to traditional methods. This is primarily because the fully personalized classifier fails to effectively leverage global consensus, thereby neglecting the critical balance between generalization and personalization.
\begin{figure}[!t]
	\centering
	\includegraphics[width=3.1in]{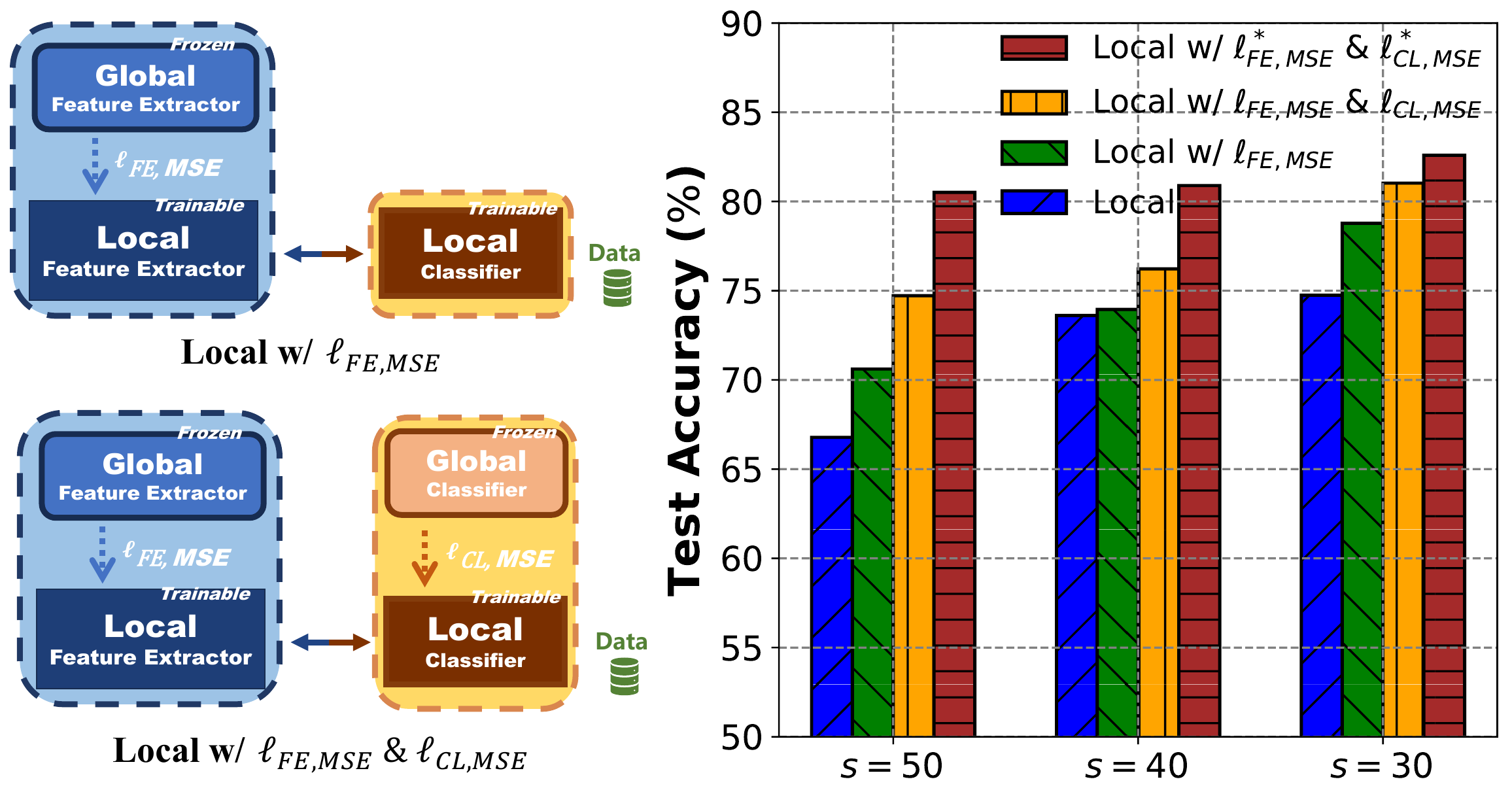}
	\caption{Performance evaluation on the Fashion-MNIST dataset \cite{ref55} under varying levels of heterogeneity, where a smaller value of $s$ indicates stronger heterogeneity. Tests are conducted under four conditions: 1) local training only, 2) feature extractor guided by global information ($\ell_{FE,MSE}$), 3) both feature extractor and classifier guided ($\ell_{FE,MSE}\;\&\;\ell_{CL,MSE}$), and 4) both optimally guided ($\ell_{FE,MSE}^{*}\;\&\;\ell_{CL,MSE}^{*}$). (Left: framework, right: results)} 
	\label{fig13}
    \vspace{-15pt}
\end{figure}

Additionally, we observe that applying strong global constraints to the feature extractor while personalizing the classifier with local information can disrupt model coherence, misalign components, and degrade overall performance. Since the local feature extractor and local classifier are trained as a unified model, an adaptive relationship exists between them, enabling the feature extraction style to align effectively with the classification decision boundary \cite{ref65}.
However, due to data heterogeneity, variations in feature extraction styles and quality across clients can lead to inconsistencies between global and local feature extractors \cite{ref15}. This discrepancy may result in a mismatch between global feature extraction and local classification decisions. To address this, it is crucial to consider both the local adaptability of the global feature extractor and the global generalization capability of the local classifier, thereby enhancing compatibility between the two components.
As demonstrated in Fig.~\ref{fig13}, our results show that applying strong global constraints solely to the feature extractor provides limited performance improvement. In contrast, incorporating global guidance for both feature extraction and classification significantly enhances performance. The model’s effectiveness can be substantially improved by achieving a reasonable coordination of global constraints across these components. Therefore, in this paper, our primary objective is to balance model generalization and personalization while fostering the coupling and coordination between the decoupled components, which are trained through differentiated strategies.

To this end, we propose FedeCouple, a novel approach designed to enhance the local adaptability of the feature extractor while preserving a high-quality feature space. Concurrently, it improves the global generalization capability of classifiers without compromising their effectiveness on personalized tasks. 
{Achieving this balance between generality and specificity is critical for real-world applications; for instance, healthcare models must transfer across hospitals yet adapt to individual patients, IoT systems need to capture global patterns while remaining device-aware, and mobile applications should learn shared representations while reflecting unique user behaviors.} 
Specifically, our method employs a frozen global feature extractor for feature extraction and abstraction. It generates privacy-protected anchors to constrain local training, fostering a feature space with strong intra-class cohesion and inter-class separation for each client. The local extractor is then updated with the latest global parameters and evaluated through both a locally trained classifier and the frozen global classifier. This dual mechanism enables the simultaneous learning of global and local feature representations.  Meanwhile, the personalized classifier is refined via dynamic knowledge distillation, guided by locally adapted global knowledge to transfer generalization capabilities from the global to the local classifier. 
By comprehensively managing the interplay between the feature extractor and the classifier, FedeCouple achieves a fine-grained balance between generalization and personalization. 

The main contributions of this paper are summarized below: 

\begin{itemize}
     \item We propose a novel model decoupling-based PFL method that enables each decoupled component to learn from other clients while adapting to local tasks. This achieves a fine-grained balance between model generalization and personalization. Specifically, the extractor integrates local and global feature information through collaborative training, while dynamic knowledge distillation enhances the generalization capability of the personalized classifier.
	\item We introduce feature anchors constructed using a frozen global extractor, which maintains a high-quality feature space characterized by strong intra-class cohesion and inter-class separation. This approach also ensures privacy preservation and controlled transmission costs. 
	\item We conduct extensive experiments on five benchmark datasets under three statistically heterogeneous settings to evaluate the effectiveness, stability, scalability, and security of the proposed method. The experimental results demonstrate that our method consistently outperforms nine baselines across all evaluated conditions. 
\end{itemize}

The rest of this paper is organized as follows: Section~\ref{rw} reviews related work, and Section~\ref{pp} introduces the preliminaries and problem formulation. The proposed approach is detailed in Section~\ref{me}, along with its theoretical analysis. Section~\ref{ex} reports and discusses the experimental results. Finally, Section~\ref{cf} concludes the paper.

\section{Related Work}\label{rw}
Statistical heterogeneity is a prevalent issue in data from MNs, posing a critical challenge for FL research. Heterogeneity refers to substantial distributional differences among the data of FL participants, meaning that data maintained by different clients is non-IID. Traditional FL methods often fail to adequately address the non-IID problem, leading to a significant degradation in model performance under heterogeneous data conditions. For instance, in terms of gradient update direction, the imbalance in training tasks can cause a substantial deviation between the global gradient and the expected optimal direction \cite{ref26}. Additionally, from the perspective of feature space, heterogeneous data can lead to dimensional collapse in the output model \cite{ref27}. To tackle the challenge of statistical heterogeneity, various approaches have been proposed, which we categorize into three main types: 1)~Global aggregation optimization, ensuring fairer combination through accurate assessment of local model quality, as detailed in Section~\ref{3A}; 2) local training modification, implementing PFL to balance model generalization and personalization, as discussed in Section~\ref{3B}; 3) data enrichment, enhancing the model's ability to comprehensively understand the data, as elaborated in Section~\ref{3C}. 
{For clarity, Tab.~\ref{tab:rw_summary} provides a comparative summary of state-of-the-art (SOTA) FL methods and FedeCouple, outlining their key characteristics.}

\begin{table*}[!t]
\centering
\caption{Comparison analysis of SOTA FL methods against FedeCouple.
Symbols: \ding{51} strong, $\triangle$ moderate, \ding{55} weak.}
\label{tab:rw_summary}
\resizebox{\textwidth}{!}{%
\begin{tabular}{>{\centering\arraybackslash}m{5.6cm}
                >{\centering\arraybackslash}m{6.7cm}
                >{\centering\arraybackslash}m{6.4cm}}
\toprule
\textbf{Category \& SOTA Methods} & \textbf{Strengths} & \textbf{Limitations} \\
\midrule

\begin{minipage}{\linewidth}\centering
\textbf{Global Aggregation Optimization}\\[1pt]
{\footnotesize (FedAvg~\cite{ref9}, FedAMP~\cite{ref61}, Krum~\cite{blanchard2017machine}, FedFomo~\cite{zhang2021personalized}, Elastic Aggregation~\cite{chen2023elastic})}
\end{minipage}
&
\begin{minipage}{\linewidth}\centering
\ding{51}\;Performs fair aggregation via quality weighting\\[1pt]
\ding{51}\;Enhances robustness by filtering unreliable updates
\end{minipage}
&
\begin{minipage}{\linewidth}\centering
\ding{55}\;High computational and communication overhead\\[1pt]
$\triangle$\; Privacy risks by distributing statistical information
\end{minipage}
\\
\midrule

\begin{minipage}{\linewidth}\centering
\textbf{Local Training Modification}\\[1pt]
{\footnotesize (FedProx~\cite{ref13}, Ditto~\cite{ref37}, FedKD~\cite{wu2022communication}, FedProto~\cite{ref23}, FedPAC~\cite{ref46})}
\end{minipage}
&
\begin{minipage}{\linewidth}\centering
\ding{51}\;Harmonizes inconsistent feature spaces across clients\\[1pt]
\ding{51}\;Coarsely balances generalization and personalization
\end{minipage}
&
\begin{minipage}{\linewidth}\centering
\ding{55}\;Overconfident classification decisions\\[1pt]
$\triangle$\; Additional cost and privacy leakage threat from prototype sharing
\end{minipage}
\\
\midrule

\begin{minipage}{\linewidth}\centering
\textbf{Data Enrichment}\\[1pt]
{\footnotesize (Generation~\cite{ref64, esser2021taming}, Augmentation~\cite{ref50,ref51})}
\end{minipage}
&
\begin{minipage}{\linewidth}\centering
\ding{51}\;Mitigates data distribution disparities\\[1pt]
\ding{51}\;Reinforces intrinsic data knowledge learning
\end{minipage}
&
\begin{minipage}{\linewidth}\centering
\ding{55}\;High computational burden due to generative model\\[1pt]
$\triangle$\; Limited efficacy of augmentation for unseen classes
\end{minipage}
\\
\midrule

\rowcolor{gray!10}
\begin{minipage}{\linewidth}\centering
\vspace{2ex}\textbf{FedeCouple (Ours)}
\end{minipage}
&
\begin{minipage}{\linewidth}\centering
\ding{51}\;Maintains flexible yet stable feature representations\\[1pt]
\ding{51}\;Accommodates transferability with adaptability\\[1pt]
\ding{51}\;Reduces communication while preserving privacy
\end{minipage}
&
\begin{minipage}{\linewidth}\centering
\vspace{1ex}$\triangle$\; Slightly higher per-round computational demand to achieve a superior performance trade-off
\end{minipage}
\\
\bottomrule
\end{tabular}}
\vspace{-10pt}
\end{table*}

\vspace{-10pt}
\subsection{Global Aggregation Optimization} 
\label{3A}
Researchers have proposed a range of global aggregation optimization strategies to address non-IID data. When client data exhibits statistical heterogeneity, differences in data distributions can lead to task shifts, causing local models to develop divergent capabilities~\cite{ref26}. If a naive aggregation method, such as FedAvg~\cite{ref9}, is employed---where aggregation weights are determined solely based on the data size of each client---the global model may fail to effectively integrate the strengths of all local models, resulting in issues akin to catastrophic forgetting. 
Leveraging statistical information related to client data distributions \cite{ref23} or feature representations/logits from local model training \cite{chen2023the} can help establish more reasonable aggregation weights. But such approaches may also increase the risk of exposing client privacy. While the collaboration of large and small models can partially enable efficient integration of local model capabilities \cite{ref67,cheng2022fedgems}, the substantial computational and communication costs associated with these methods significantly increase implementation complexity. 
{Motivated by these trade-offs, we aim for a global aggregation strategy that is both unbiased and effective in building consensus, with privacy as a core constraint. By adhering to the ``minimal auxiliary information" principle, we reduce reliance on side information and emphasize low-leakage aggregation signals. This approach enhances task adaptability and cross-distribution generalization while minimizing privacy exposure~\cite{Ullah,preserving}.}

In the context of lightweight global models, accurately evaluating the quality of local models followed by fair aggregation---without relying on additional aggregation information---presents an effective approach. Quality evaluation is typically based on differences between models, using similarities in capabilities or parameters of local models to determine aggregation weights. Representative works in this area include FedAMP~\cite{ref61}, Krum \cite{blanchard2017machine}, FedFomo \cite{zhang2021personalized}, elastic aggregation \cite{chen2023elastic}, 
which aim to obtain a robust global model by filtering out low-quality or malicious local models. However, directly calculating the deviation between each client and others incurs high computational costs and can reduce the effectiveness of weight assignment. To address this, we propose a two-step approach for model evaluation and fair aggregation. Specifically, we treat local models as a set of points in a high-dimensional space. First, we determine the centroid by averaging the points. Then, we calculate the distance between each point and the centroid to assign the aggregation weights. This method is computationally efficient and ensures fairness in the weight assignment process. 
{Moreover, the design philosophy of our two-step aggregation strategy aligns with the mobility-aware clustered FL framework proposed by Feng et al.~\cite{feng2022mobility}, which emphasizes balancing global aggregation and local training through a two-tier hierarchical wireless network scheme. This shared emphasis on hierarchical coordination highlights its critical role in reinforcing the robustness and flexibility of global learning in dynamic environments.}

\vspace{-5pt}
\subsection{Local Training Modification}\label{3B}
The objective of local training modification is to strike a balance between model generalization and personalization, effectively harmonizing both its universality and specificity. PFL integrates global consensus knowledge into local training \cite{ref34}. A primary challenge in this process is defining personalized components and effectively utilizing global information. 

\subsubsection{Model Decoupling} 
Building on prior knowledge from multi-task learning \cite{ref24}, various approaches have been proposed to decouple models into feature extractors and classifiers, utilizing distinct updating and training strategies. One type of method attempts to keep the feature extractor local while only aggregating the classifier~\cite{ref40}. 
However, maintaining a consistent feature space across clients under heterogeneous-label settings can be challenging, potentially undermining the strong transferability of the feature extractor, as emphasized in multi-task and transfer learning~\cite{ref23}. Additionally, some methods impose strong global constraints on the feature extractor while fully localizing the classifier~\cite{ref38,ref39}. These approaches often fail to adequately address classifier consensus, which may lead to coordination and compatibility issues between the two components of the model. 
Furthermore, several approaches attempt to decouple features and process them using different classifiers \cite{ref41}. For model decoupling-based approaches, careful consideration is still required to effectively combine classifiers with functional differences. 
In particular, achieving a finer balance between generalization and individuality for the decoupled components is crucial. This balance would enable the model to integrate universality and specialization effectively.

\subsubsection{Global Guidance}
When guiding local training, the form of global information and its guidance method significantly influence the performance of the output model. Category centroids of feature representations (i.e., prototypes/anchors) serve as a simple yet effective carrier of information. For instance, FedPAC \cite{ref46} and FedProto \cite{ref23} utilize prototypes to constrain the training of local feature extractors, FedProc \cite{mu2023fedproc} employs prototypes for contrastive learning, and FedPCL \cite{tan2022federated} leverages prototypes to learn personalized models. However, these methods require uploading local prototypes and the sample size of each class, which not only increases communication costs but also raises the risk of compromising user privacy.  
Regarding guidance techniques, some methods focus on transferring model capabilities. For example, FedFA \cite{ref47} and FedPAC \cite{ref46} apply center loss in the feature space before computing the cross-entropy loss \cite{wen2016discriminative}. Similarly, FedMD \cite{li2019fedmd} and FedKD \cite{wu2022communication} employ knowledge distillation related to classification decisions. Additionally, FedProx \cite{ref13} and Ditto \cite{ref37} impose constraints on model parameters, though such approaches are limited by the model structure.  
Against this backdrop, our method employs a frozen global extractor on the client to derive prototypes from local data, which serve as anchors for the center loss to enforce feature space constraints. Simultaneously, dynamic knowledge distillation enhances the model's classification performance. 

\subsection{Data Enrichment}\label{3C}
Data directly impacts the effectiveness of FL models. Thus, efficiently enriching client data can help alleviate the issues arising from non-IID data distributions. 
Currently, research in this area can be broadly categorized into two approaches: data expansion and data augmentation.
Data expansion primarily utilizes generative techniques \cite{ref64, esser2021taming} to generate synthetic data, enabling clients to obtain additional samples and fill data gaps. However, this process often demands considerable computational resources due to the need for training advanced models and managing high-dimensional information. 
In contrast, data augmentation transforms existing data through operations like rotation, scaling, cropping, and color adjustment \cite{ref50, ref51}, enhancing the model's ability to capture intrinsic information while maintaining computational efficiency. 
{Given these considerations, 
we combine augmented data with the original data to enhance model representation learning while preserving efficiency. This strategy aligns with semi-supervised learning principles that integrate labeled and unlabeled data. For instance, SFedXL~\cite{SFedXL} leverages cross-sharpness training and layer freezing to effectively fuse supervised and unsupervised information, thereby improving model stability and scalability. Building on this philosophy, our method exploits complementary data sources to bolster intrinsic knowledge learning.}

\section{Preliminaries and Problem Formulation}\label{pp}
\subsection{Terminology}
We consider an FL scenario with $N$ clients. The overall dataset is denoted as $\mathcal{D}$, with the local dataset of the $i$-th client represented as $\mathcal{D}_i$, such that $\Sigma_{i=1}^{N}\mathcal{D}_i=\mathcal{D}$. The dataset $\mathcal{D}$ consists of a set of categories $C$, with the total number of classes being $|C|$ and the classes indexed by $[C]$. The local dataset $\mathcal{D}_i$ of the $i$-th client contains classes indexed by $[C_i]\subseteq[C]$. For a sample $(x_i, y_i)$ from the $i$-th client’s dataset, the input $x_i$ belongs to the $d$-dimensional input space $\mathcal{X}$, and the label $y_i$ is indexed by one of the classes in $[C_i]$.
Here, we use the $i$-th client as an example to illustrate the entire process of model decoupling, with the same approach applied to the other clients. We decouple the model, with parameters $\omega_i$, into two components: a feature extractor with parameters $\theta_i$ and a classifier with parameters $\phi_i$. The feature extraction function $f_i$ maps samples from the input space $\mathcal{X}$ to the feature space $\mathcal{H}$, generating a $K$-dimensional feature vector, denoted as ${f_i: \mathbb{R}^d \to \mathbb{R}^K}$. For the classification function $g_i$, it uses the $K$-dimensional feature vector to produce the final classification output, where the output dimension is determined by the number of classes $|C_i|$, represented as ${g_i: \mathbb{R}^K \to \mathbb{R}^{|C_i|}}$. The result from the classifier is a logits output, which has not yet been normalized. We define the model last layer as the classifier and all preceding layers as the feature extractor.

\subsection{Loss Function and Data Heterogeneity Setting}
\subsubsection{Loss Function}
In general FL methods, the goal is to learn a global model that can generalize well across the data distributions of all clients. This is feasible when the data distributions among clients are either identical or similar. However, the task becomes significantly more challenging when there are substantial discrepancies in the data distributions across clients. In such cases, loss function $\mathcal{L}(\omega)$ is given by, 
\begin{equation} \nonumber
	\min_\omega \mathcal{L}(\omega) = \sum_{i=1}^{N} {\alpha_i} \frac{\sum_{(x_i, y_i) \in \mathcal{D}_i} \ell_{ce}(\omega; x_i, y_i)}{|\mathcal{D}_i|}, \alpha_{i}=\frac{\left|\mathcal{D}_{i}\right|}{\Sigma_{j}\left|\mathcal{D}_{j}\right|},
\end{equation}
where $\ell_{ce}$ denotes the cross-entropy loss, and $\omega$ represents the global model. In PFL, we relax this requirement and focus on learning a set of personalized local models, where each model is optimized for the specific local task, aligning with private dataset. In this cases, loss function $\mathcal{L}(W)$ is represented as:
\begin{equation} \nonumber
	\min_{W\in\{\omega_1,\omega_2,...,\omega_N\}}\mathcal{L}(W)=\sum_{i=1}^{N}{\alpha_i}\frac{\sum_{(x_i,y_i)\in \mathcal{D}_i}\ell_{ce}(\omega_i;x_i,y_i)}{|\mathcal{D}_i|}.
\end{equation}
\subsubsection{Data Heterogeneity}
We use $\mathbb{P}$ to represent the data distribution. The distribution on the $i$-th client is given by $\mathbb{P}(x_i, y_i) = \mathbb{P}(x_i|y_i) \mathbb{P}(y_i) = \mathbb{P}(y_i|x_i) \mathbb{P}(x_i)$ where $\mathbb{P}(x_i)$ and $\mathbb{P}(y_i)$ are the marginal distributions of inputs and labels while $\mathbb{P}(x_i|y_i)$ and $\mathbb{P}(y_i|x_i)$ denote the conditional distributions. This paper focuses on label shift, where the label distributions vary across clients and the feature distributions within each label remain similar. Specifically, for two distinct clients $i$ and $j$, we have $ \mathbb{P}(y_i) \neq \mathbb{P}(y_j)$ and $\mathbb{P}(x_i | y_i) = \mathbb{P}(x_j | y_j)$.

\begin{figure*}[!h] 
	\centering 
    \includegraphics[width=0.8\textwidth]{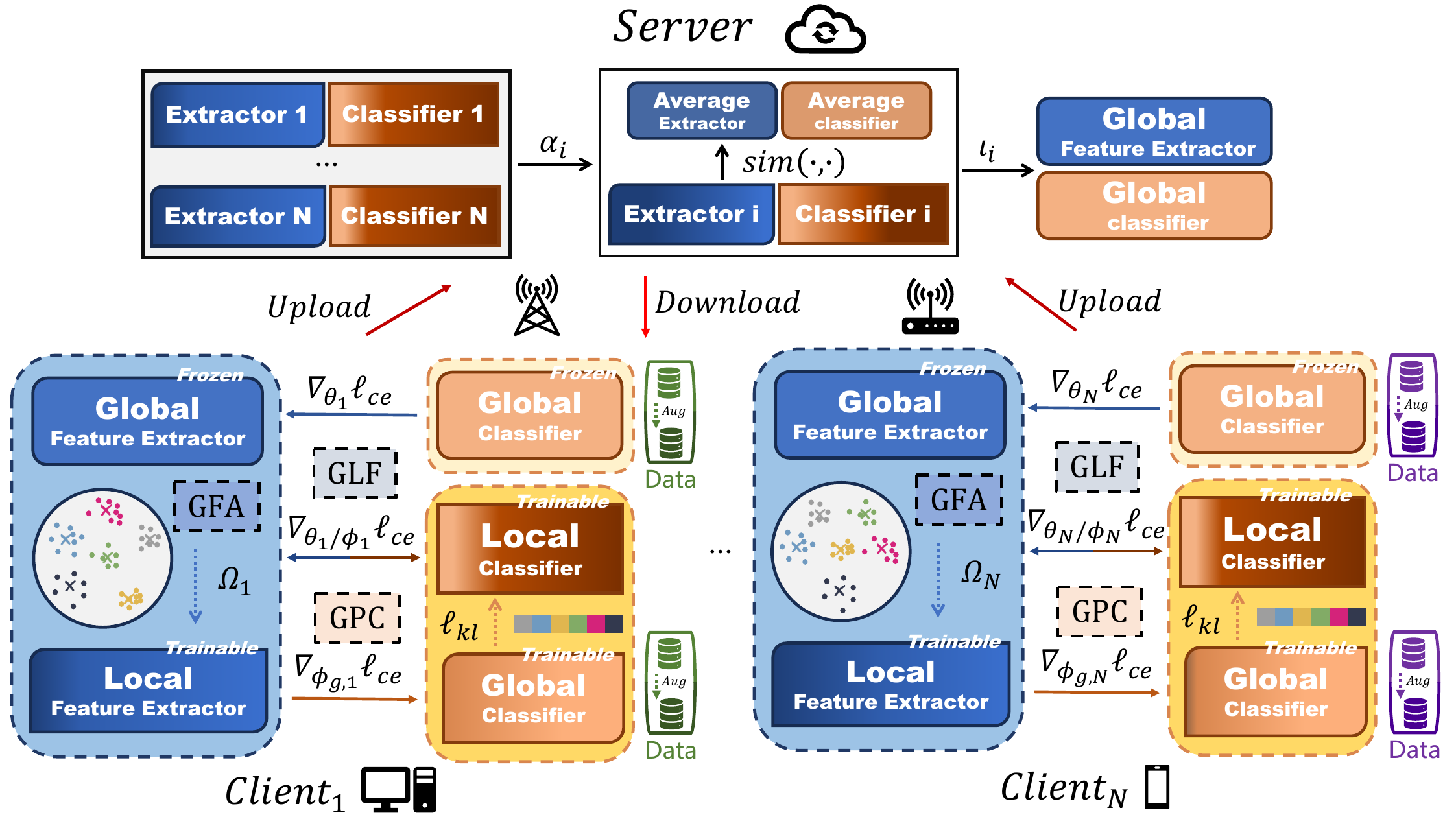}
	\caption{Workflow of the FedeCouple method. In this process, the local feature extractor is combined with both the global frozen classifier and the local training classifier, enabling global-local feature integrated learning (\textbf{GLF}). Simultaneously, the global feature anchors (\textbf{GFA}), generated by the global feature extractor, are used to constrain the model training and construct the center loss. For the local classifier, knowledge distillation is applied to transfer knowledge from the global classifier to the local classifier. This enables the local classifier to simultaneously learn both generalization and personalization in its classification capabilities (\textbf{GPC}). To achieve fairer aggregation, the similarity between local models and the global average model is calculated and used as the new aggregation weight. Additionally, to fully utilize the data, data augmentation is performed using the RandAugment technique.} 
	\label{fig4} 
    \vspace{-5pt}
\end{figure*}

\begin{algorithm}[htbp]
	\caption{FedeCouple}\label{alg:fedecouple}
	\begin{algorithmic}[1] 
		\STATE \textbf{Input:} $\theta^0, \phi^0$: initial parameters; $N$: number of clients; $\eta$: learning rate; $T$: total communication rounds;  $E_{fe}$, $E_{cl}$: extractor and classifier local epochs; $\rho$: client joining ratio; $\lambda$, $\mu$: hyperparameters.
		\STATE \textbf{Output:} personalized model $\{\omega_1^{T-1}, \omega_2^{T-1}, \dots, \omega_N^{T-1}\}$.
		\STATE $\rhd$\textit{\Large{\textbf{Server}}}
		\FOR{$t \gets 0$ \textbf{to} $T-1$}
		\STATE {Select a client subset} $\mathcal{J}^t$ {based on} $\rho$ {and} $N$
		\FOR{each $i \in \mathcal{J}^t$}
		\STATE {Send} $\omega^t \{\theta^t, \phi^t\}$ {to client}
		\STATE \textbf{\textsc{LocalTraining}}
		\STATE Get local model $\omega_i^{t+1} \{\theta_i^{t+1}, \phi_i^{t+1}\}$
		\ENDFOR
		\STATE $\omega^{t+1} \gets \text{Aggregate local model (} \mathcal{J}^t\text{)}$ base on Eq.~(\ref{eq16})
		\ENDFOR
		\STATE $\rhd$\textit{\Large{\textbf{Client}}}
		\STATE \textit{\textbf{Classifier:}}
         \STATE Update: $\phi_{g,i}^t \gets \phi^{t}$
		\FOR{$e_{cl} \gets 0$ \textbf{to} $E_{cl}-1$}
		\STATE \textit{Local Extractor Combine Local Classifier:}
		\STATE \hspace{0.5cm}Train local classifier based on Eq.~(\ref{eq12})
		(\textit{Guide by \\\hspace{0.4cm} Knowledge Distillation} Eq.~(\ref{eq11}))
		\STATE \textit{Local Extractor Combine Training Global Classifier:}
		\STATE \hspace{0.5cm}Train global classifier based on Eq.~(\ref{eq13})
		\ENDFOR
		\STATE Get local classifier $\phi_i^{t+1}$
		\STATE \textit{\textbf{Feature Extractor:}}
		\STATE Update: $\theta_{i}^{t} \gets \theta^{t}$
		\STATE \textit{Local Extractor Combine Frozen Global Classifier:}
		\STATE \hspace{0.3cm}Train local extractor one epoch based on Eq.~(\ref{eq14})
		\FOR{$e_{fe} \gets 0$ \textbf{to} $E_{fe}-1$}
		\STATE \textit{Local Extractor Combine Local Classifier:}
		\STATE \hspace{0.5cm}Train local extractor based on Eq.~(\ref{eq15}) (\textit{Guide by\\\hspace{0.4cm} Global Feature Anchors} Eq. (\ref{eq8}))
		\ENDFOR
		\STATE Get local feature extractor $\theta_i^{t+1}$
		\STATE \textbf{Return} $\omega_i^{t+1} \{\theta_i^{t+1} ,\phi_i^{t+1}\}$ \textbf{to Server}
	\end{algorithmic}
\end{algorithm}

\section{Methodology}\label{me}
Based on the discussion in Section~\ref{intro}, our goal is to achieve a balance between model generalization and personalization through the flexible coordination of the feature extractor and classifier. 
Our method is depicted in Fig.~\ref{fig4} and Alg.~\ref{alg:fedecouple}.
For the local feature extractor, we update it using the global extractor before training, and enhance the feature space quality by leveraging the localized global feature anchors. At the same time, the local extractor is combined with both global and local classifiers, facilitating the integrated learning of global-local feature information.
For the personalized local classifier, we employ knowledge distillation to constrain classifier training, utilizing globally applicable knowledge tailored for local adaptation. 
Additionally, we perform fair aggregation by assigning weights based on similarity to the center of the local models. 
To further enhance the model's learning capability, we incorporate data augmentation, enabling the model to learn essential knowledge from the data effectively. 
\subsection{Local Feature Extractor Training}\label{4a}
\subsubsection{Integrated Learning of Global-Local Features (GLF)}
Before initiating local training of the feature extractor, we update it using the most recent global feature extractor ($\theta^t$). We then combine this updated feature extractor with the latest global classifier, whose parameters remain fixed during local training, referred to as the frozen global classifier. Subsequently, we train the parameters of the feature extractor: 
\begin{equation}\label{eq3} \nonumber
	\theta_i^t\leftarrow\theta^t,\; \theta_i^t\leftarrow\theta_i^t-\eta\nabla_{\theta_i}\left[\ell_{ce}(\theta_i^t,\phi_{fro}^t;\gamma_i^t)+\mu\Omega_i^t\right],
\end{equation}
where $t$ represents the $t$-th round of local training, $(x_i,y_i)\in\gamma_i$, with $\gamma_i$ representing mini-batches from $\mathcal{D}_i$, $\eta$ is the learning rate, $\phi_{fro}^t$ indicates the frozen global classifier, $\Omega$ is  the regularization term related to global information, and $\mu$ is a hyperparameter.
Further, we combine the feature extractor with the local classifier and continue training the feature extractor: 
\begin{equation}\label{eq5} \nonumber
	\theta_i^t\leftarrow\theta_i^t-\eta\nabla_{\theta_i}\left[\ell_{ce}(\theta_i^t,\phi_i^t;\gamma_i)+\mu\Omega_i^t\right].
\end{equation}

\begin{figure}[!t]
	\centering
	\includegraphics[width=3.4in]{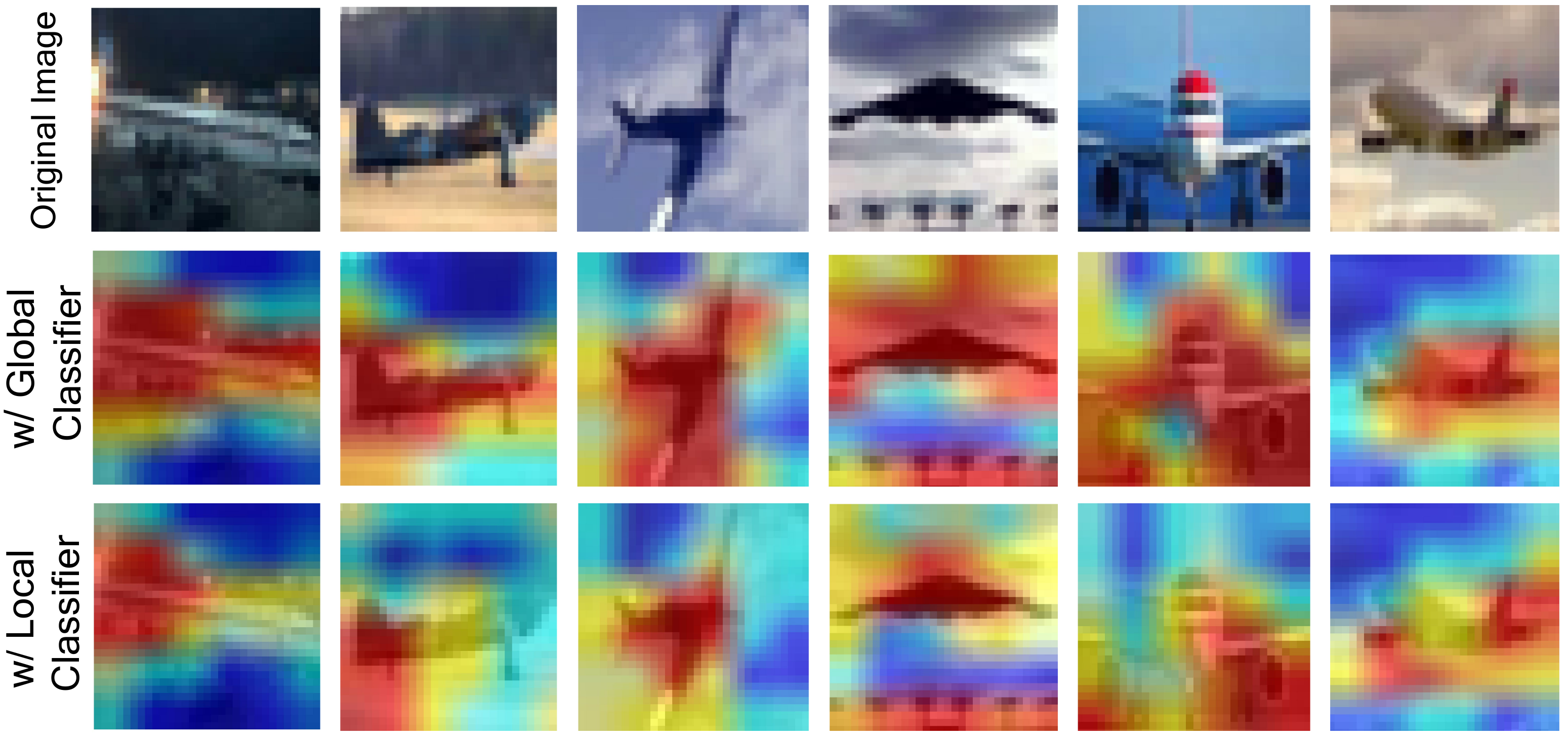}
	\caption{Feature attention regions of the final convolutional layer on CINIC-10\cite{ref54} by Grad-CAM \cite{ref62}. This visualization is performed in two scenarios: First, by combining the feature extractor with the global classifier (w/ Global Classifier); second, by pairing the feature extractor with the personalized local classifier (w/ Local Classifier).}
    \vspace{-10pt}
	\label{fig1}
\end{figure}
\textit{Attention Regions in the Final Convolutional Layer:} 
To better understand how feature extraction varies when the extractor is paired with either the global or local classifier, we use the Grad-CAM \cite{ref62} technique to visualize the attention regions in the final convolutional layer.  
{As shown in Fig.~\ref{fig1}, the attention heatmaps reveal distinct patterns based on the classifier pairing. When coupled with the global classifier, the model distributes attention broadly, emphasizing overall shapes and structural boundaries, such as the airplane’s silhouette. This indicates a focus on transferable features that facilitate cross-client generalization. Conversely, when paired with the local classifier, attention becomes more concentrated, highlighting discriminative regions like the wing edges, fuselage, or tail. These localized features capture subtle variations that help distinguish objects with similar global structures, thereby enhancing client-specific performance.
Together, these observations underscore the synergistic roles of the global and local classifiers: the former promotes transferable insights across clients, while the latter refines discriminative details for personalization. This dual mechanism enables a balanced integration of shared and task-specific knowledge.}


\subsubsection{Center Loss via Global Feature Anchors (GFA)} 
We use the frozen global extractor, which remains unchanged during training, to extract features from the local data. These features are then averaged by class to derive the feature anchor $An$:
\begin{equation}\label{eq6} \nonumber An_k^t=\sum_{x_{i,k}\in\mathcal{D}_{i,k}}\frac{f(\theta^t;x_{i,k})}{|\mathcal{D}_{i,k}|},
\end{equation}
where $\mathcal{D}_{i,k}$ denotes the subset of samples belonging to class $k$ from the dataset of the $i$-th client, with $k\in[C]$, and $x_{i,k}$ represents an individual sample drawn from this subset. $f(\theta^t; x_{i,k})$ indicates the feature representation extracted by the frozen global extractor (similarly, $f_i(\theta^t_i; x_{i,k})$ indicates the representation obtained from the training local extractor).
Further, we apply center loss $\Omega_i^t$ based on global anchors to regularize the training of the local extractor, thereby improving the representations prior to the final classification, 
\begin{equation}\label{eq8}
	\Omega_i^t = \sum_{x_{i,k} \in \mathcal{D}_{i,k}} \mathbb{L}_2^2\left( A n_k^t, f_i\left( \theta_i^t; x_{i,k} \right) \right).
\end{equation}

\textit{Representations in Feature Space}: 
{Center loss with global feature anchors enhances feature space quality by pulling same-class representations toward their anchors, thereby reducing intra-class variability. Inter-class separation is maintained through sufficient distances between anchors. This results in a structured feature space characterized by compact, well-separated clusters, which yields more discriminative and robust representations.} 

{To evaluate this effect, we aggregate feature vectors from all clients during the testing phase and visualize them using t-SNE~\cite{ref69}. This analysis is performed solely for evaluation purposes; no feature vectors are exchanged during training. As shown in Fig.~\ref{fig2}, the feature space of FedRep \cite{ref39} appears scattered with blurred inter-class boundaries. In contrast, our anchor-based constraints produce compact intra-class clusters centered around their anchors and exhibit clearer inter-class separation. 
These results indicate that global anchors enhance both the discriminability and robustness of learned representations while simultaneously promoting cross-client consistency. 
Additionally, since our anchors are derived locally from client data, they retain client-specific characteristics while aligning with global representations. Unlike global anchor aggregation methods~\cite{ref23,ref46}, our approach avoids transmitting anchors or class-specific data, thereby reducing computational overhead and mitigating privacy risks. Relevant experimental results are detailed in Section~\ref{privacy}.}
\begin{figure}[!t] 
	\centering 
	\includegraphics[width=3.4in]{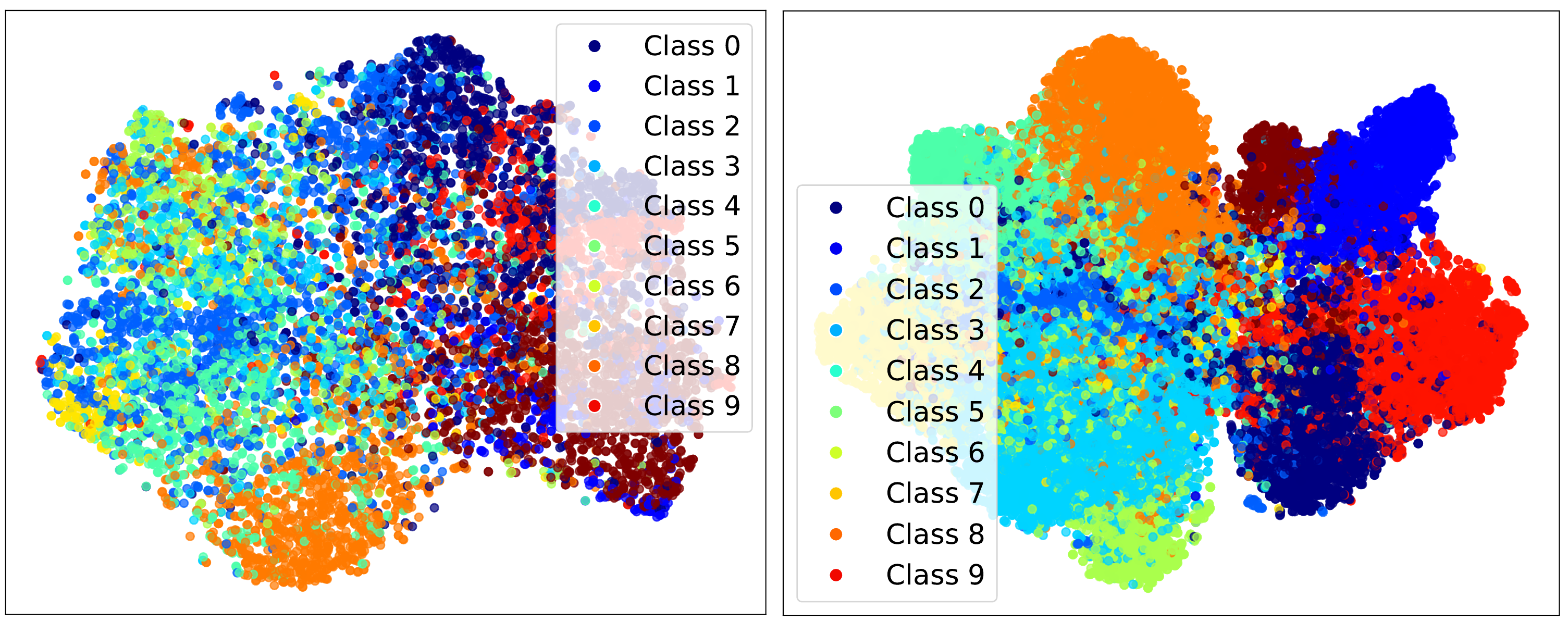} 
	\caption{Representations visualization on  CIFAR-10\cite{ref53} by t-SNE\cite{ref69} ($\beta = 5$). The left panel shows the feature space obtained with FedRep\cite{ref39}, while the right illustrates the feature space after applying center loss with global feature anchors.} 
	\label{fig2} 
\end{figure}

\subsection{Local Classifier Training}\label{4b}
\subsubsection{Knowledge Distillation for Generalized and Personalized Classification (GPC)}
To improve the model's adaptation to local tasks, we employ the personalized update, retaining the local classifier from the previous training round. During training, the local representations are input into both the global and local classifiers, producing the corresponding logits $Lt$:
\begin{equation} \nonumber
\phi^t_{g,i}\leftarrow\phi^t, Lt_{i}^t=g_i(\phi_i^t;f_i(\theta_i^t;x_{i})),\: Lt^t_{g,i}=g(\phi^t_{g,i};f_i(\theta^t_i;x_{i})),
\end{equation}
where $g_i(\phi_i^t;\cdot)$ and $g(\phi^t_{g,i};\cdot)$ (update by global classifier $\phi^t$) denote the decisions generated by the training local and global classifiers. Then, the logits is scaled using a temperature coefficient $\tau$, followed by normalization with softmax \cite{ref70}:
\begin{equation} \nonumber
	p_{T_i}^t=softmax(Lt^t_{g,i}/\tau),\: p_{S_i}^t=softmax(Lt_{i}^t/\tau).
\end{equation}
Further, the normalized probability distributions are used to measure the discrepancy between global and local classifiers through Kullback-Leibler (KL) divergence. This divergence is incorporated into the training process of the local classifier, while the global classifier is also fine-tuned locally:
\begin{equation*} 
	\phi_i^t \leftarrow \phi_i^t-\eta\nabla_{\phi_i}\left[\ell_{ce}(\theta_i^t,\phi_i^t;\gamma_i^t) + \lambda \ell_{kl}(p_{T_i}^t,p_{S_i}^t)\right],
\end{equation*}
\begin{equation}\label{eq11}
	\phi^t_{g,i} \leftarrow \phi^t_{g,i}-\eta\nabla_{\phi_{g,i}}\left[\ell_{ce}(\theta_i^t,\phi^t_{g,i};\gamma_i^t)\right],
\end{equation}
where $\lambda$ is a hyperparameter, and $\ell_{kl}$ is the KL divergence.

\textit{Performance Evaluation of Classification:}
We utilize the knowledge from the global classifier to guide the training of the local classifier. The global consensus enhances local model's generalization capability, while knowledge distillation ensures a more stable training process.
We incorporate knowledge distillation from the global classifier into the FedRep method and evaluate model accuracy before and after local training. As shown in Fig.~\ref{fig3}, after applying consensus-based constraints, classification performance improves, and the accuracy curve becomes smoother and more consistent.
Additionally, fixing the global classifier consensus as static knowledge may hinder the optimization of local classifier performance. To mitigate this, we perform local adaptation of the global classifier, fostering a balance between generalization and specialization in classification.
\begin{figure}[!t]
	\centering
	\includegraphics[width=3.2in]{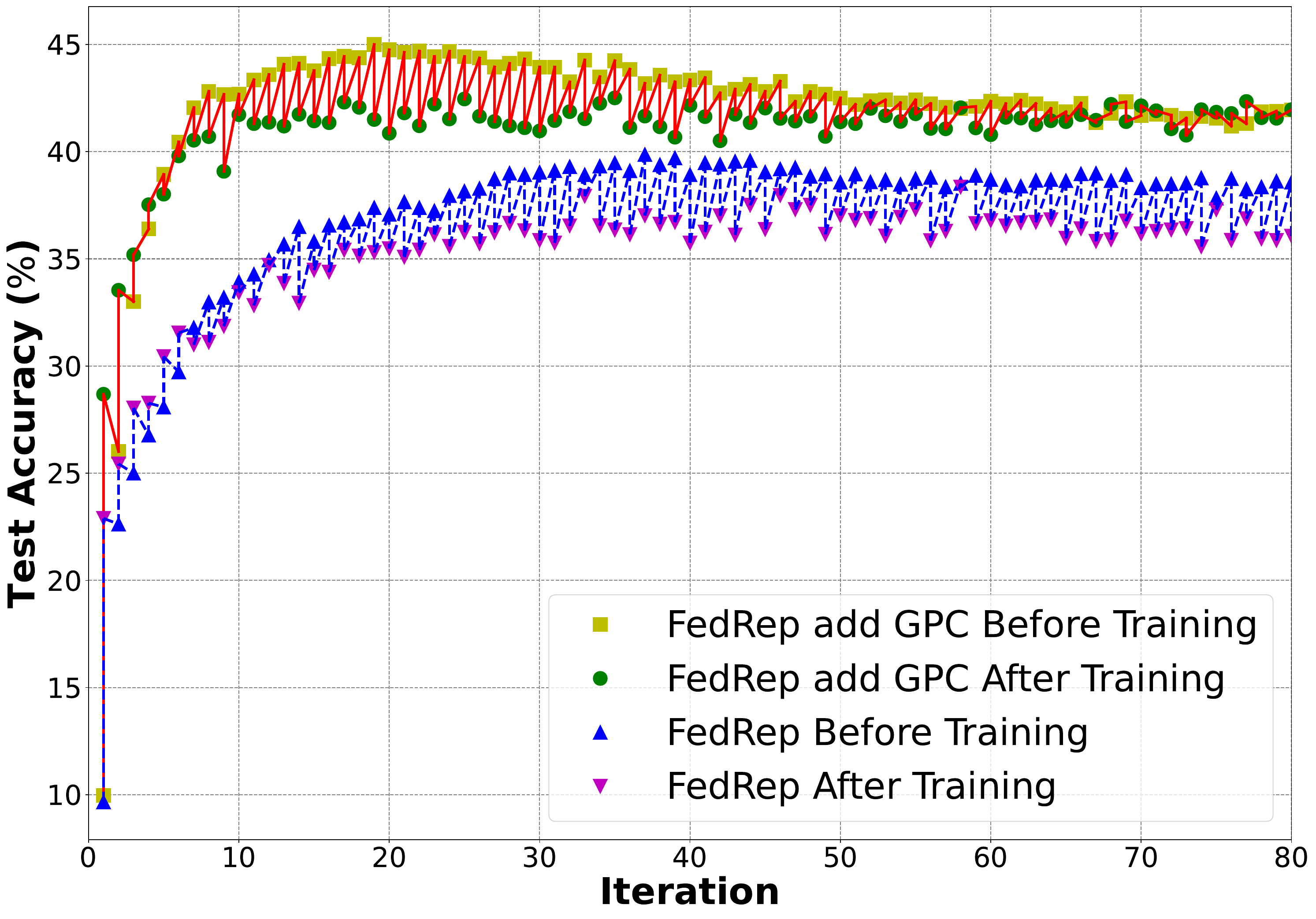}
	\caption{Test accuracy for classifier knowledge distillation on the CINIC-10 dataset~\cite{ref54} with $\beta=1$. The blue dashed line represents the FedRep\cite{ref39} method, with upward and downward triangles indicating accuracy before and after local training, respectively. The red solid line shows the effect of applying GPC knowledge distillation, where squares and circles represent accuracy before and after local training, respectively.}
	\label{fig3}
\end{figure}

Below is a concise overview of the local training process, in which the feature extractor and classifier are trained separately.
Initially, we train both the local and global classifiers according to Eq.~\ref{eq12} and Eq.~\ref{eq13}, incorporating the knowledge distillation loss $\ell_{kl}$ into the local classifier's training. 
Subsequently, we train the local feature extractor in combination with the frozen global classifier and the training local classifier, as described in Eq.~\ref{eq14} and Eq.~\ref{eq15}, while center loss $\Omega$ is introduced.
\begin{equation}\label{eq12}
	\phi_{i}^{t} \leftarrow \phi_{i}^{t} - \eta \nabla_{\phi_{i}}\left[ \ell_{ce}(\theta_{i}^{t}, \phi_{i}^{t}; \gamma_{i}^t) + \lambda\ell_{kl}(p_{T_i}^t,p_{S_i}^t)\right], 
\end{equation}
\begin{equation}\label{eq13}
	\phi^{t}_{g,i} \leftarrow \phi^{t}_{g,i} - \eta \nabla_{\phi_{g,i}}\left[ \ell_{ce}(\theta_{i}^{t}, \phi^{t}_{g,i}; \gamma_{i}^t)\right], 
\end{equation}
\begin{equation}\label{eq14}
	\theta_{i}^{t} \leftarrow \theta_{i}^{t} - \eta \nabla_{\theta_{i}}\left[ \ell_{ce}(\theta_{i}^{t}, \phi^{t}_{fro}; \gamma_{i}^t) + \mu \Omega_i^t\right],
\end{equation}
\begin{equation}\label{eq15}
	\theta_{i}^{t} \leftarrow  \theta_{i}^{t} - \eta \nabla_{\theta_{i}}\left[ \ell_{ce}(\theta_{i}^{t}, \phi_{i}^{t+1}; \gamma_{i}^t) + \mu \Omega_i^t\right].
\end{equation}

{{\textbf{Remark on the Trade-offs.}}}
{FedeCouple is designed to simultaneously address effectiveness, communication efficiency, and privacy. 
To this end, it employs a global feature extractor to generate anchors locally, eliminating the communication of raw data or data-derived anchors and thus reducing overhead and privacy risks. 
For effectiveness, FedeCouple emphasizes the complementary coordination of its decoupled components, integrating global knowledge into local adaptation to achieve both fine-grained generalization and client-specific personalization.
Crucially, these objectives are not isolated trade-offs but mutually reinforcing: enhanced communication efficiency and privacy protection facilitate more effective and secure adaptation of global knowledge, resulting in a synergistic balance among accuracy, efficiency, and confidentiality.} 

\subsection{Global Aggregation Scheme}\label{4c}
Due to issues such as low-quality data and the non-IID problem, designing aggregation weights based solely on the amount of client data may not be fair (e.g., when a large dataset contains duplicate or low-quality data). 
Therefore, we aim to calculate the similarity between the local model and the global average model (which reduces computational complexity by an order of magnitude compared to calculating similarities between local models) and use this as a measure of model quality to set more equitable aggregation weights. Specifically,
\begin{equation}\label{eq16}
\omega_{avg}^t=\sum_{i=1}^N\alpha_i\omega_i^t,\: \iota_i^t=\frac{sim(\omega_{i}^t,\omega_{avg}^t)}{\sum_j sim(\omega_{j}^t,\omega_{avg}^t)},\omega^t=\sum_{i=1}^N{\iota}_l^t\omega_l^t,
\end{equation}
where $\omega_{avg}^t$ denotes the global average model, $sim$ indicates the cosine similarity, $\iota_i^t$ refers to the aggregation weights based on model quality, and $\omega^t$ represents the global model, which consists of global feature extractor $\theta^t$ and global classifier $\phi^t$.

\subsection{Data Augmentation}\label{4d}
Our method utilizes RandAugment~\cite{ref51} for data augmentation, chosen for its low computational cost and effectiveness in the knowledge distillation process. Recent studies indicate that incorporating challenging augmented data significantly enhances knowledge transfer from the teacher model to the student model~\cite{ref52}, thereby improving learning performance and model generalization. 
Accordingly, we adopt five specific augmentation techniques, illustrated in Fig.~\ref{fig12}: image cropping, flipping, rotation, brightness adjustment, and pixel-value inversion. During augmentation, we randomly select one of these techniques along with its transformation intensity to introduce controlled variations into the training images. 
These augmentations yield more diverse and robust feature representations, effectively mitigating overfitting. By simulating real-world variations in factors such as lighting, orientation, and background, they also enhance the model's ability to generalize to previously unseen scenarios.
\begin{figure}[h!]
	\centering
	\includegraphics[width=3.5in]{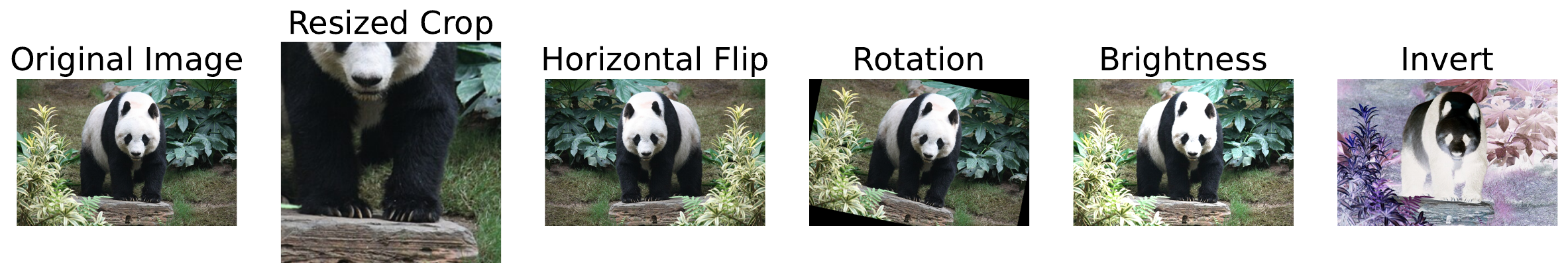}
	\caption{Demonstration of the effect of data augmentation.}
	\label{fig12}
\end{figure}

\subsection{Convergence Analysis}\label{4e}
Leveraging insights from FedProto~\cite{ref46} and FedFA~\cite{ref47}, we conduct a convergence analysis for the proposed method. The loss function of the local model is defined as follows:
\begin{equation} 
	\mathcal{L}_{i}^{t}=\ell_{ce}(\theta_{i}^{t}, \phi_{i}^{t}; \gamma_{i}^t)+\lambda\ell_{kl}(p_{T_i}^t,p_{S_i}^t)+\mu\Omega_i^t.
\end{equation}

Before presenting the theoretical results, we first establish the following assumptions: 
\begin{assumption}\label{as1}
(Lipschitz Smooth) For all $t_1,t_2>0$ and $i\in\{1,2,...,N\}$, the objective function is $L_1$-smooth:
\begin{equation} \nonumber
\mathcal{L}_{i}^{t_{1}}-\mathcal{L}_{i}^{t_{2}}\leq\langle\nabla \mathcal{L}_{i}^{t_{2}},\left(\omega_{i}^{t_{1}}-\omega_{i}^{t_{2}}\right)\rangle+\frac{L_{1}}{2}\left\|\omega_{i}^{t_{1}}-\omega_{i}^{t_{2}}\right\|_{2}^{2}.
\end{equation}
\end{assumption} 

\begin{assumption}\label{as2} 
(Unbiased Gradient and Bounded Variance) The expectation of the gradient obtained from a single sample is equal to the gradient obtained from full sampling:
\begin{equation} \nonumber
\mathbb{E}_{\gamma_{i}^{t}\sim \mathcal{D}_{i}}[gr_{i}^{t}]=\nabla\mathcal{L}^t_i, \quad \mathbb{E}\left[\left\|gr_i^t-\nabla\mathcal{L}_i^t\right\|_2^2\right]\leq\sigma^2.\end{equation}
\end{assumption}

\begin{assumption}\label{as3} 
(Bounded Expectation of Euclidean Norm of Single-Sample Gradients) The expectation of $\mathbb{L}_2$-norm of the single-sample gradient is bounded by 
\begin{equation} \nonumber
\mathbb{E} \left[\left\|gr_i^t\right\|_2\right] \leq G.
\end{equation}
\end{assumption}

\begin{assumption}\label{as4} 
(Lipschitz Continuity) The loss function of the feature extractor is $L_2$-Lipschitz continuous:
\begin{equation} \nonumber
\left\| f_{i}\left(\theta_{i}^{t_{1}}\right)-f_{i}\left(\theta_{i}^{t_{2}}\right)\right\|_{2}^{2}\leq L_{2}\left\|\theta_{i}^{t_{1}}-\theta_{i}^{t_{2}}\right\|_{2}^{2}.
\end{equation}
\end{assumption}

Assumptions~\ref{as1}-\ref{as3} also apply to the training of the global classifier. Similarly, Assumption~\ref{as4} applies to the updating of the classifiers. 

The total number of training iterations is denoted by $t$, with $E$ local epochs between each global communication. After the $k$-th global round, $t=kE$. Both server aggregation and local updates occupy half a segment of time. Consequently, the training progress during local epochs between global rounds is represented by $e\in\{1/2,1,2,3,..., E\}$. The interval $[0,1/2]$ is further divided into $[0,1/4]$ (local model update) and $[1/4,1/2]$ (global information update).

We begin by analyzing the variation of the loss function throughout the local model training process. 
\begin{lemma}\label{le1}
If the local model is trained via stochastic gradient descent (SGD), we have 
\begin{align*}
    \mathbb{E}\left(\mathcal{L}_i^{(t+1)E}\right)
    - \mathcal{L}_i^{tE+1/2} \leq\ & L_1 \eta^2 \left(E\sigma^2 + \sum_{e=1/2}^{E-1} \left\|\nabla \mathcal{L}_i^{tE+e}\right\|_2^2 \right) \nonumber \\
    & - \eta \sum_{e=1/2}^{E-1} \left\|\nabla \mathcal{L}_i^{tE+e}\right\|_2^2.
\end{align*}
\end{lemma}
Next, we analyze loss function variation during the local model update process. 
\begin{lemma}\label{le2} 
The variation of the loss function during the local model update process is bounded by
\begin{align*}
    \mathbb{E}\left(\mathcal{L}_i^{(t+1)E+\frac{1}{4}}\right)
    &- \mathcal{L}_i^{(t+1)E} \leq\  -\eta \sum_{i=1}^{N} \iota_{i}^{tE} \sum_{e=\frac{1}{2}}^{E-1} \nabla \mathcal{L}_{i}^{tE+e} \nonumber \\
    & + \eta \sum_{e=\frac{1}{2}}^{E-1} \nabla \mathcal{L}_{i}^{tE+e} + 2 \lambda \eta^{2} L_{2}^2 E G^{2} |\mathcal{D}_{i}|.
\end{align*}
\end{lemma}
Furthermore, we analyze the loss function variation during the global information update process. 
\begin{lemma}\label{le3}
The variation of the loss function during the global information update process is bounded by
\begin{align*}
    \mathbb{E}\left(\mathcal{L}_{i}^{(t+1)E+\frac{1}{2}}\right) 
    - \mathbb{E}\left(\mathcal{L}_{i}^{(t+1)E+\frac{1}{4}}\right) \leq& \lambda \eta^2 L_2^2 E G^2 |\mathcal{D}_i| \nonumber \\
    & + \frac{2 \mu \eta L_2 E G}{\tau}.
\end{align*}
\end{lemma}
Now, we establish our final theoretical result. 
\begin{theorem}
Suppose that Assumptions~\ref{as1}-\ref{as4} hold. The variation of the loss function between two adjacent communication rounds is bounded by 
\begin{equation} \label{eq_thm1}
	\begin{split}
		\ & \mathbb{E}\left(\mathcal{L}_i^{(t+1)E+1/2}\right) - \mathcal{L}_i^{tE+1/2} \leq\eta\left(\Gamma_1+\Gamma_2\right),
	\end{split}
\end{equation}
where $\Gamma_2 =  L_1 \eta E \sigma^2 + 3\lambda \eta L_2^2 E G^2 |\mathcal{D}_i| +  \frac{2 \mu L_2 E G}{\tau} $ and $\Gamma_1 = \sum_{e=1/2}^{E-1} \left( L_1 \eta \left\|\nabla \mathcal{L}_i^{tE+e}\right\|_2^2 - \left\|\nabla \mathcal{L}_i^{tE+e}\right\|_2^2 -\sum_{i=1}^N \iota_i^{tE} \nabla \mathcal{L}_i^{tE+e}  \right.  \\
\left. +\nabla \mathcal{L}_i^{tE+e} \right)$  .
If the following conditions are satisfied:
\begin{equation}
	\eta_{e^{\prime}}<\frac{\sum_{e=1/2}^{e^{\prime}}\Gamma_{\mathcal{L}} - \frac{2\mu L_{2}EG}{\tau}}{\sum_{e=1/2}^{e^{\prime}}L_1\left\|\nabla\mathcal{L}_i^{tE+e}\right\|_2^2+\Gamma_3},\: \mu_{e^{\prime}}<\frac{\tau\sum_{e=1/2}^{e^{\prime}}\Gamma_{\mathcal{L}}}{2L_{2}EG},
\end{equation} 
where $\Gamma_{\mathcal{L}}=\left\|\nabla\mathcal{L}_i^{tE+e}\right\|_2^2 + \sum_{i=1}^N \iota_i^{tE} \nabla \mathcal{L}_i^{tE+e}- \nabla\mathcal{L}_i^{tE+e} $, $\Gamma_3=L_1 E \sigma^2 + 3\lambda L_2^2 EG^2 |\mathcal{D}_i|, e^{\prime}=1/2, 1, ..., E-1$, 
we have $\eta\left(\Gamma_1+\Gamma_2\right) < 0$. This ensures a decrease in the loss function with each round, ultimately leading to convergence. 
\end{theorem}

\begin{proof}
    See Appendices~A and~B for details.
\end{proof}

\section{Experiment}\label{ex}
\subsection{Experimental Setup and Implementation Details}\label{exper}
\subsubsection{Benchmark Datasets}
We conduct experiments using five datasets: CIFAR-10\cite{ref53}, CIFAR-100\cite{ref53}, Fashion-MNIST (FMNIST)\cite{ref55}, Extended MNIST (EMNIST)\cite{ref56}, and CINIC-10\cite{ref54}. 
Specifically, CIFAR-10 consists of 60,000 32×32 color images across 10 classes, while CIFAR-100 contains the same number of images but spans 100 classes. CINIC-10 includes 270,000 images from ImageNet and CIFAR-10 distributed across 10 classes. FMNIST comprises 70,000 28×28 grayscale images representing 10 clothing categories. EMNIST extends the original MNIST dataset with 814,255 images across 62 classes. 
These diverse datasets facilitate a more accurate and comprehensive evaluation of model performance.
\subsubsection{Model Architecture}
In our experiments, we employ two distinct neural network models, each designed for specific datasets. For the CIFAR-10, CIFAR-100, and CINIC-10 datasets, we use the first model, which comprises three convolutional layers with 16, 32, and 64 filters, respectively. The first two convolutional layers have a kernel size of 5, while the third uses a kernel size of 3. Each convolutional layer is followed by a Leaky ReLU activation function \cite{ref57} and a max-pooling operation. The fully connected layers include an intermediate layer with 128 units and an output layer corresponding to the number of classes. For the FMNIST and EMNIST datasets, we we adopt the second model, which shares a similar structure but is adapted for single-channel images. It consists of two convolutional layers with 16 and 32 filters, each followed by max-pooling and Leaky ReLU activations. The fully connected layers also include an intermediate layer with 128 units, leading to an output layer corresponding to the number of classes. 
In both models, the final layer serves as the classifier, while the preceding layers function as the feature extractor. 

\subsubsection{Default Training Settings}
We use mini-batch SGD as the local optimizer for all methods. The feature extractor and classifier are trained locally for 5 epochs across 200 global communication rounds. The hyperparameters are set to $\lambda = 0.8$ and $\mu = 2.0$, with a temperature coefficient of $T = 2$. The global aggregation rate is 100\%, involving 20  participating clients, where each client’s testing data follows the same distribution as its training data. We report the average test accuracy (\%) across clients. 
{Results are reported as the mean $\pm$ standard deviation over five independent runs. The ``Gain" metric denotes the performance improvement over the best-performing baseline. Statistical significance, indicated by \textsuperscript{\dag}, is assessed using a two-sided paired $t$-test against the best baseline ($p{<}0.05$, $df{=}4$).}

\subsubsection{Comparative Methods}
Local only: each client trains its model independently; FT-FedAvg: the global model, aggregated using FedAvg~\cite{ref9}, is fine-tuned on the local dataset; FedProx~\cite{ref13}: global information constrains the update magnitude during local training; 
Ditto~\cite{ref37}: clients train personalized models while using the global model as a regularization term. We also compare with model decoupling methods: LG-FedAvg~\cite{ref40} aggregates the classifier while personalizing the feature extractor; 
FedPer~\cite{ref38} and FedRep~\cite{ref39} while using the global model as a regularization term. FedRep alternates between training the extractor and the classifier, a strategy our method also adopts. 
FedBABU~\cite{ref22} aggregates only the feature extractor and fine-tunes the classifier locally. FedPAC~\cite{ref23} assigns personalized aggregation weights to each client's classifier on the server.
\vspace{-5pt}
\subsection{Accuracy Comparison Across Varying Data Settings}
In this subsection, we conduct extensive experiments to compare the proposed FedeCouple algorithm with several state-of-the-art methods across three data heterogeneity settings: Weak Pathological Setting ($s$), Practical Setting ($\beta$), 
and Pathological Setting.  
To enhance clarity, Fig.~\ref{fig11} provides a detailed breakdown of the data distribution across these settings, offering a visual representation of how data is distributed under different conditions. 
\begin{figure}[htbp]
	\centering
	\includegraphics[width=3.4in]{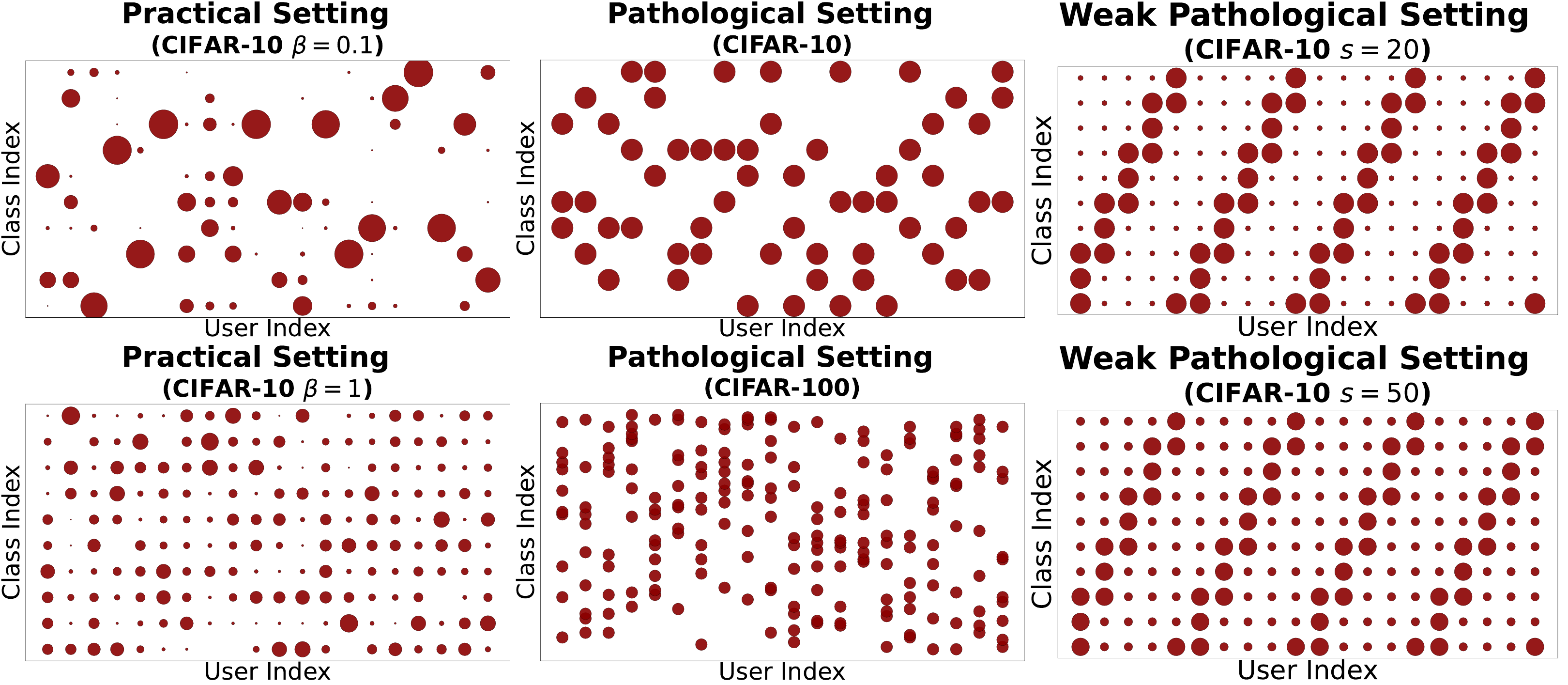}
	\caption{Data distribution visualization across three data heterogeneity settings.}
	\label{fig11}
    \vspace{-8pt}
\end{figure}
\begin{table}[htbp]
\centering
\caption{Test accuracy under weak pathological setting.}
\label{table2}
\renewcommand{\arraystretch}{1.1}    
\setlength{\tabcolsep}{3pt}          
\LARGE                                
\begin{adjustbox}{max width=\columnwidth}
\begin{tabular}{lcccccc}
\toprule
\multirow{2}{*}{\textbf{Method}} & \multicolumn{2}{c}{\textbf{CIFAR-10}} & \textbf{CIFAR-100} & \multicolumn{2}{c}{\textbf{CINIC-10}} & \textbf{FMNIST} \\
 & $s{=}20$ & $s{=}50$ & $s{=}20$ & $s{=}20$ & $s{=}50$ & $s{=}20$ \\
\midrule
\textit{Local} & \textit{64.3$\pm$0.1} & \textit{51.0$\pm$0.2} & \textit{63.4$\pm$0.1} & \textit{50.4$\pm$0.1} & \textit{43.0$\pm$0.1} & \textit{85.5$\pm$0.1} \\
FT-FedAvg      & 76.9$\pm$0.2 & 70.2$\pm$0.1 & 70.4$\pm$0.1 & 57.0$\pm$0.1 & 51.7$\pm$0.2 & 90.3$\pm$0.1 \\
FedProx        & 77.0$\pm$0.1 & 71.3$\pm$0.2 & 73.2$\pm$0.1 & 58.2$\pm$0.1 & 51.3$\pm$0.1 & 90.2$\pm$0.2 \\
Ditto          & 75.7$\pm$0.1 & 69.8$\pm$0.1 & 73.3$\pm$0.1 & 56.4$\pm$0.1 & 51.2$\pm$0.1 & 90.3$\pm$0.1 \\
FedPer         & 68.4$\pm$0.1 & 58.3$\pm$0.2 & 66.5$\pm$0.1 & 51.2$\pm$0.1 & 47.1$\pm$0.1 & 87.2$\pm$0.1 \\
LG-FedAvg      & 64.8$\pm$0.2 & 51.3$\pm$0.1 & 70.7$\pm$0.2 & 51.2$\pm$0.1 & 43.8$\pm$0.2 & 85.7$\pm$0.1 \\
FedRep         & 71.1$\pm$0.1 & 62.3$\pm$0.1 & 67.0$\pm$0.1 & 53.7$\pm$0.2 & 48.0$\pm$0.1 & 88.2$\pm$0.1 \\
FedPAC         & 81.1$\pm$0.1 & 75.9$\pm$0.1 & 76.3$\pm$0.2 & 59.3$\pm$0.1 & 53.4$\pm$0.1 & 91.2$\pm$0.2 \\
FedBABU        & 76.5$\pm$0.1 & 71.3$\pm$0.2 & 66.6$\pm$0.2 & 52.0$\pm$0.1 & 46.7$\pm$0.1 & 90.0$\pm$0.2 \\
\midrule
\textbf{Ours}  
& \textbf{83.4$\pm$0.2$^{\dag}$}  
& \textbf{77.9$\pm$0.1$^{\dag}$}  
& \textbf{78.4$\pm$0.2$^{\dag}$}  
& \textbf{61.8$\pm$0.2$^{\dag}$}  
& \textbf{57.7$\pm$0.1$^{\dag}$}  
& \textbf{92.3$\pm$0.1$^{\dag}$}   \\
Gain  
& +2.3 & +2.0 & +2.1 & +2.6 & +4.3 & +1.1 \\
$t$-test / $p$-value 
& 3.3 / 0.03 
& 3.0 / 0.04 
& 3.1 / 0.04 
& 3.4 / 0.03 
& 3.5 / 0.03 
& 2.9 / 0.04 \\
\bottomrule
\end{tabular}
\end{adjustbox}
\end{table}

\subsubsection{Weak Pathological Setting} Following the approaches SCAFFOLD~\cite{ref60} and FedAMP~\cite{ref61}, we ensure that all clients possess an equal amount of data, typically 600 samples. A portion, $s\%$, is uniformly sampled from all classes, with each class receiving $(|\mathcal{D}| \times s\%)/|C|$ samples, where $|C|$ denotes the number of classes and $|\mathcal{D}|$ represents the total dataset size. The remaining $(100 - s)\%$ is assigned to randomly selected dominant classes. As $s$ decreases, the degree of heterogeneity among clients increases. 

As shown in Tab.~\ref{table2}, methods based on fully personalized classifiers (e.g., FedRep and FedPer) or fully personalized feature extractors (e.g., LG-FedAvg) do not achieve optimal performance. This suggests that relying solely on the personalization of a single model component may be insufficient to address the data heterogeneity across clients. 
Additionally, we observe that fine-tuning is particularly effective. Fine-tuned FedAvg consistently demonstrates robust performance across most scenarios, indicating that localized adjustments to the global model can enhance overall accuracy. 
In contrast, our approach updates the extractor using global information while preserving personalized information in the classifier. By incorporating the GLF and GPC components, we improve the adaptability of feature extraction and the generalization of classification. Furthermore, we employ data augmentation techniques to enable the model to capture the underlying data patterns. Experimental results validate these observations, with our method consistently outperforming the baselines. On CINIC-10, when $s=50$, our method exceeds the best baseline by $4.3\%$.

\begin{table}[htbp]
    \centering
    \caption{Test accuracy under pathological setting.}
    \renewcommand{\arraystretch}{1.2} 
    \setlength{\tabcolsep}{5pt} 
    \begin{adjustbox}{max width=\textwidth}
        \begin{tabular}{lcccc}
            \toprule
             \textbf{Method} & \textbf{CIFAR-10} & \textbf{CINIC-10} & \textbf{FMNIST} & \textbf{CIFAR-100} \\
            \midrule
            \textit{Local}   & \textit{88.37$\pm$0.1} & \textit{81.22$\pm$0.2} & \textit{95.60$\pm$0.1} & \textit{71.93$\pm$0.2} \\
            FT-FedAvg        & 90.62$\pm$0.2 & 84.71$\pm$0.1 & 95.45$\pm$0.2 & 75.23$\pm$0.1 \\
            FedProx          & 90.77$\pm$0.1 & 84.34$\pm$0.2 & 95.31$\pm$0.1 & 74.17$\pm$0.1 \\
            Ditto            & 91.24$\pm$0.1 & 84.28$\pm$0.1 & 95.28$\pm$0.2 & 76.04$\pm$0.1 \\
            FedPer           & 91.34$\pm$0.2 & 85.43$\pm$0.1 & 95.13$\pm$0.2 & 74.33$\pm$0.1 \\
            LG-FedAvg        & 88.05$\pm$0.1 & 81.93$\pm$0.1 & 95.17$\pm$0.1 & 72.75$\pm$0.2 \\
            FedRep           & 89.17$\pm$0.2 & 82.64$\pm$0.1 & 95.26$\pm$0.2 & 75.56$\pm$0.1 \\
            FedPAC           & 92.06$\pm$0.1 & 85.66$\pm$0.1 & 96.22$\pm$0.1 & 78.16$\pm$0.2 \\
            FedBABU          & 90.29$\pm$0.1 & 86.03$\pm$0.1 & 95.63$\pm$0.2 & 75.27$\pm$0.1 \\
            \midrule
            \textbf{Ours}    & \textbf{93.66$\pm$0.1$^{\dag}$} & 
                               \textbf{87.33$\pm$0.2$^{\dag}$} & 
                               \textbf{97.33$\pm$0.1$^{\dag}$} & 
                               \textbf{80.51$\pm$0.2$^{\dag}$} \\
            Gain             & +1.6 & +1.3 & +1.1 & +2.4 \\
            $t$-test / $p$-value 
            & 4.0 / 0.015 
            & 3.5 / 0.025 
            & 3.2 / 0.034 
            & 3.3 / 0.029 \\
            \bottomrule
        \end{tabular}
    \end{adjustbox}
    \label{table3}
\end{table}

\subsubsection{Pathological Setting} In this heterogeneous scenario, each client has access to a subset of classes. For example, on the CIFAR-10, CINIC-10, and FMNIST datasets, each client holds data from 3 classes, while on CIFAR-100, each client has data from 10 classes. This configuration provides a more rigorous evaluation of the model’s performance on local tasks.

As shown in Tab.~\ref{table3}, Ditto, FedPAC, and FedBABU perform well on datasets with fewer classes, such as CIFAR-10, FMNIST, and CINIC-10, but face challenges with the more complex CIFAR-100 dataset. In this scenario, our method outperforms the best baseline by 2.4\% in accuracy, demonstrating its superior ability to adapt to high heterogeneity and effectively manage local personalized training tasks. 

\begin{table*}[htbp]
\centering
\caption{Test accuracy under practical setting.}
\label{practical}
\renewcommand{\arraystretch}{1.25}
\setlength{\tabcolsep}{6pt} 
\Large 
{\LARGE
\begin{adjustbox}{max width=\textwidth}
\begin{tabular}{lccccccccccccccc}
\toprule
\multirow{2}{*}{\textbf{Method}} &
\multicolumn{4}{c}{\textbf{CIFAR-10}} &
\multicolumn{4}{c}{\textbf{CIFAR-100}} &
\multicolumn{2}{c}{\textbf{CINIC-10}} &
\multicolumn{1}{c}{\textbf{EMNIST}} &
\multicolumn{2}{c}{\textbf{FMNIST}} \\
\cmidrule(lr){2-5}\cmidrule(lr){6-9}\cmidrule(lr){10-11}\cmidrule(lr){12-12}\cmidrule(lr){13-14}
& $\beta=0.01$ & $\beta=0.1$ & $\beta=0.5$ & $\beta=1$
& $\beta=0.01$ & $\beta=0.1$ & $\beta=0.5$ & $\beta=1$
& $\beta=0.1$ & $\beta=1$
& $\beta=1$
& $\beta=0.1$ & $\beta=1$ \\
\midrule
\textit{Local}   & 63.59$\pm$0.1 & 71.12$\pm$0.1 & 60.59$\pm$0.1 & 56.38$\pm$0.2 & 51.52$\pm$0.1 & 39.36$\pm$0.1 & 16.84$\pm$0.1 & 12.30$\pm$0.1 & 73.81$\pm$0.1 & 45.87$\pm$0.1 & 84.80$\pm$0.1 & 95.14$\pm$0.2 & 81.87$\pm$0.1 \\
FT\text{-}FedAvg & 75.45$\pm$0.2 & 86.96$\pm$0.1 & 71.78$\pm$0.2 & 67.71$\pm$0.1 & 49.93$\pm$0.1 & 33.64$\pm$0.1 & 25.96$\pm$0.2 & \textbf{24.78$\pm$0.1} & 75.48$\pm$0.1 & 52.61$\pm$0.1 & 93.68$\pm$0.1 & 93.36$\pm$0.2 & 88.18$\pm$0.1 \\
FedProx          & 77.64$\pm$0.1 & 86.80$\pm$0.2 & 70.90$\pm$0.1 & 65.18$\pm$0.1 & 51.07$\pm$0.1 & 44.51$\pm$0.1 & 24.54$\pm$0.1 & 22.48$\pm$0.1 & 80.58$\pm$0.1 & 51.61$\pm$0.1 & 94.16$\pm$0.1 & 95.02$\pm$0.1 & 88.98$\pm$0.1 \\
Ditto            & 79.54$\pm$0.2 & 84.36$\pm$0.1 & 71.05$\pm$0.2 & 64.55$\pm$0.1 & 53.86$\pm$0.1 & \textbf{45.27$\pm$0.1} & 25.79$\pm$0.1 & 24.28$\pm$0.1 & 78.82$\pm$0.1 & 51.13$\pm$0.1 & 92.86$\pm$0.1 & 95.65$\pm$0.1 & 89.23$\pm$0.1 \\
FedPer           & 86.55$\pm$0.1 & 82.53$\pm$0.2 & 62.50$\pm$0.1 & 57.90$\pm$0.1 & 48.63$\pm$0.1 & 43.69$\pm$0.1 & 14.39$\pm$0.1 & 12.19$\pm$0.1 & 69.53$\pm$0.2 & 44.79$\pm$0.2 & 88.40$\pm$0.1 & 96.13$\pm$0.1 & 85.08$\pm$0.2 \\
LG\text{-}FedAvg & 76.05$\pm$0.2 & 80.29$\pm$0.1 & 58.45$\pm$0.1 & 52.25$\pm$0.1 & 43.11$\pm$0.1 & 35.96$\pm$0.1 & 18.28$\pm$0.1 & 12.74$\pm$0.1 & 80.56$\pm$0.2 & 44.19$\pm$0.2 & 86.19$\pm$0.1 & 94.59$\pm$0.2 & 82.03$\pm$0.2 \\
FedRep           & 83.21$\pm$0.1 & 82.58$\pm$0.1 & 62.16$\pm$0.2 & 57.91$\pm$0.1 & 44.15$\pm$0.1 & 41.48$\pm$0.1 & 16.03$\pm$0.1 & 11.84$\pm$0.1 & 82.67$\pm$0.1 & 46.56$\pm$0.1 & 88.73$\pm$0.1 & 91.78$\pm$0.1 & 85.68$\pm$0.1 \\
FedPAC           & 87.84$\pm$0.1 & 86.72$\pm$0.1 & 76.61$\pm$0.1 & 73.77$\pm$0.1 & 47.21$\pm$0.1 & 38.96$\pm$0.2 & 23.84$\pm$0.1 & 17.96$\pm$0.1 & 81.11$\pm$0.2 & {56.83$\pm$0.1} & {94.80$\pm$0.1} & 97.08$\pm$0.1 & 89.83$\pm$0.2 \\
FedBABU          & 84.35$\pm$0.2 & 83.62$\pm$0.1 & 63.52$\pm$0.2 & 57.78$\pm$0.1 & 53.86$\pm$0.1 & 43.08$\pm$0.1 & 16.99$\pm$0.1 & 16.43$\pm$0.1 & 76.92$\pm$0.2 & 45.26$\pm$0.2 & 89.28$\pm$0.1 & 95.64$\pm$0.1 & 85.57$\pm$0.2 \\
\midrule
\textbf{Ours}    & \textbf{90.25$\pm$0.1\textsuperscript{\dag}} &
\textbf{90.06$\pm$0.1\textsuperscript{\dag}} &
\textbf{79.32$\pm$0.1\textsuperscript{\dag}} &
\textbf{74.89$\pm$0.1\textsuperscript{\dag}} &
\textbf{54.55$\pm$0.1\textsuperscript{\dag}} &
44.23$\pm$0.1 &
\textbf{27.79$\pm$0.1\textsuperscript{\dag}} &
23.85$\pm$0.2 &
\textbf{83.91$\pm$0.1\textsuperscript{\dag}} &
\textbf{59.51$\pm$0.1\textsuperscript{\dag}} &
\textbf{95.25$\pm$0.1\textsuperscript{\dag}} &
\textbf{97.24$\pm$0.1\textsuperscript{\dag}} &
\textbf{90.43$\pm$0.1\textsuperscript{\dag}} \\
Gain & +2.41 & +3.10 & +2.71 & +1.12 & +0.69 & -- & +1.83 & -- & +1.24 & +2.68 & +0.45 & +0.16 & +0.60 \\
$t$-test / $p$-value & 4.8 / 0.01 & 5.0 / 0.01 & 4.7 / 0.01 & 4.0 / 0.02 & 3.4 / 0.03 & -- & 4.6 / 0.01 & -- & 3.9 / 0.02 & 5.1 / 0.01 & 3.1 / 0.04 & 3.0 / 0.04 & 3.3 / 0.03 \\
\bottomrule
\end{tabular}
\end{adjustbox}
}
\vspace{-5pt}
\end{table*}

\subsubsection{Practical Setting} To simulate more realistic heterogeneous data distributions, we employ Dirichlet sampling, $Dir(\beta)$. The parameter $\beta$ controls the degree of heterogeneity in the data distribution, where smaller values indicate higher heterogeneity. We apply varying levels of heterogeneity across five different datasets to evaluate the adaptability of our method more comprehensively. 

As shown in Tab.~\ref{practical}, on the CIFAR-10 and CIFAR-100 datasets, methods such as FedPer, FedRep, LG-FedAvg, and FedBABU perform poorly with $\beta=0.5$ and $\beta=1.0$. Specifically, FedPer, FedRep, and LG-FedAvg exhibit lower accuracies than local training alone and show signs of overfitting. In contrast, our method, which integrates knowledge distillation and frozen global information, mitigates these issues. In most scenarios with lower heterogeneity, it achieves results comparable to traditional FL approaches.

\begin{table}[H]
\centering
\caption{Test accuracy (\%) on FEMNIST \& Shakespeare (80/100 clients).}
\label{tab:real_world_results}
\footnotesize
\setlength{\tabcolsep}{1.1pt}
\begin{tabular}{lcccc}
\toprule
\textbf{Method} & \textbf{FEMNIST (80)} & \textbf{FEMNIST (100)} & \textbf{Shakesp. (80)} & \textbf{Shakesp. (100)} \\
\midrule
FT-FedAvg   & 78.2$\pm$0.06 & 77.3$\pm$0.05 & 38.4$\pm$0.08 & 37.2$\pm$0.07 \\
FedProx     & 79.0$\pm$0.07 & 78.2$\pm$0.04 & 39.2$\pm$0.06 & 38.1$\pm$0.09 \\
Ditto       & 81.1$\pm$0.05 & 80.3$\pm$0.06 & 41.0$\pm$0.04 & 39.8$\pm$0.06 \\
FedPer      & 80.6$\pm$0.04 & 79.5$\pm$0.05 & 40.4$\pm$0.07 & 39.1$\pm$0.08 \\
FedRep      & 82.3$\pm$0.06 & 81.7$\pm$0.05 & 42.1$\pm$0.03 & 40.7$\pm$0.06 \\
FedPAC      & 84.2$\pm$0.05 & 83.5$\pm$0.04 & 43.0$\pm$0.05 & 42.1$\pm$0.07 \\
FedBABU     & 83.3$\pm$0.04 & 82.7$\pm$0.05 & 43.2$\pm$0.04 & 42.0$\pm$0.05 \\
\textbf{Ours} & \textbf{85.4$\pm$0.08} & \textbf{84.6$\pm$0.03} & \textbf{44.5$\pm$0.05} & \textbf{43.7$\pm$0.04} \\
\bottomrule
\end{tabular}
\end{table}

\subsubsection{{Evaluation on Real-World Federated Benchmarks}}
{We further evaluate our method on two real-world federated benchmarks, FEMNIST and Shakespeare~\cite{caldas2019leaf}, using 80 and 100 clients to assess its adaptability to naturally occurring user heterogeneity. Unlike prior experiments based on synthetic partitions (e.g., Dirichlet or pathological sampling), these datasets capture real behavioral differences across users, leading to more complex and unstructured distribution shifts. 
As shown in Tab.~\ref{tab:real_world_results}, our method consistently outperforms baselines under both settings and maintains stable generalization as heterogeneity increases. We attribute this robustness to the synergy of two core components: anchor-guided supervision enables precise modeling of client-specific structures, while dynamically updated teacher distillation mitigates overfitting and model drift induced by non-IID data. Together, these mechanisms foster strong generalization and resilience in realistic federated environments.}

\subsection{Scalability and Stability Evaluation}
\begin{figure}[htbp]
	\centering
	\includegraphics[width=3.1in]{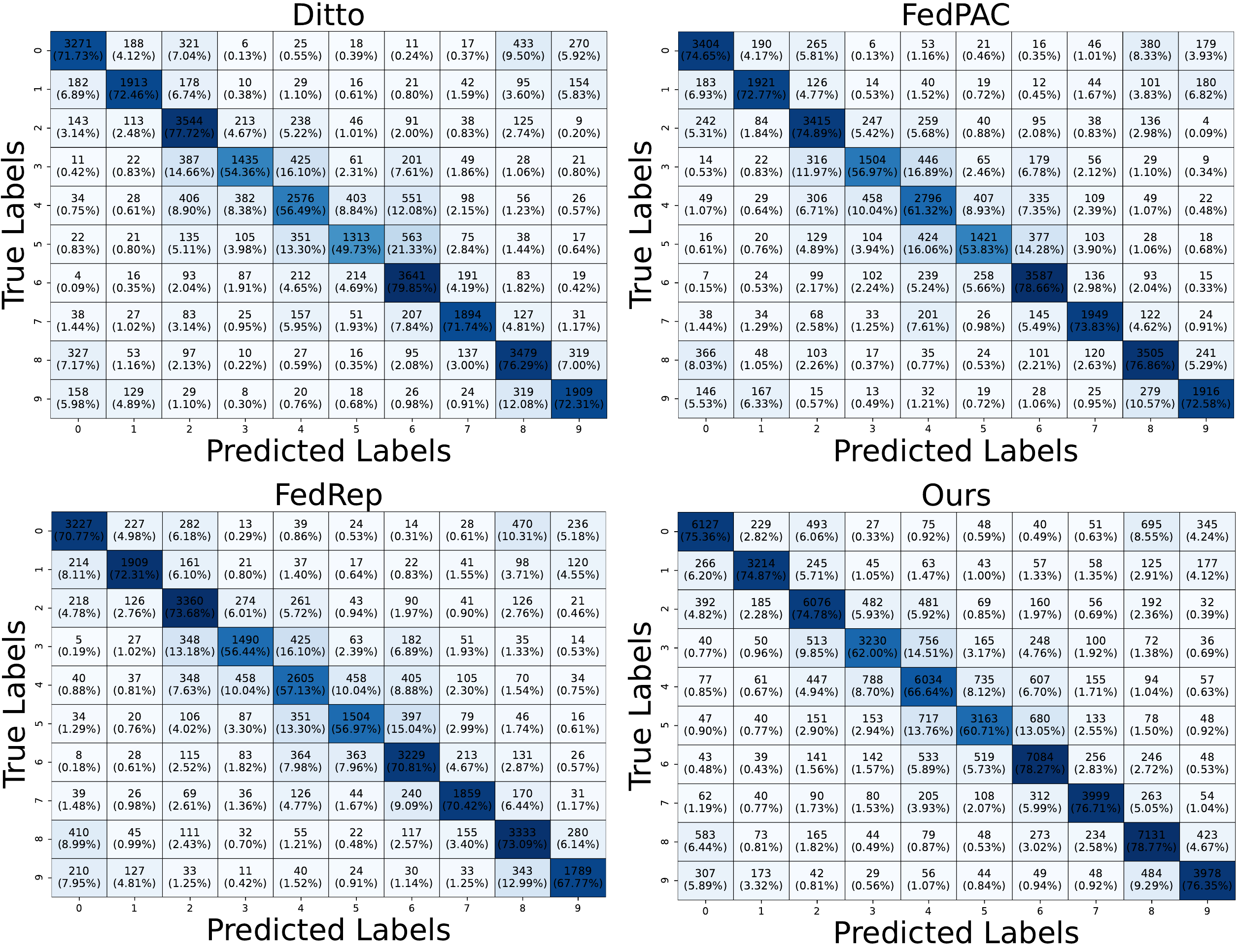}
	\caption{Prediction performance of different methods across classes across the CINIC-10 dataset with $s=20$.}
    \vspace{-10pt}
	\label{fig6}
\end{figure}
\subsubsection{Prediction Performance across Categories}
We evaluate the accuracy of various methods across categories on the CINIC-10 dataset with $s=20$. 
As shown in Fig.~\ref{fig6}, our method achieves the highest classification accuracy, as indicated by the most intense color on the diagonal of the confusion matrix, demonstrating superior performance across all categories. 
Notably, we apply data augmentation to local data, effectively expanding the training set without violating FL principles, as no external data is introduced---only transformations of the original dataset. 
In comparison, other methods exhibit higher error rates in confusion-prone categories, such as those with indices 3, 4, and 5 (i.e., cat, deer, and dog in CINIC-10). Our model better captures the intrinsic features of these images, leading to improved accuracy. 

\begin{figure*}[!t]
	\centering
	\includegraphics[width=7in]{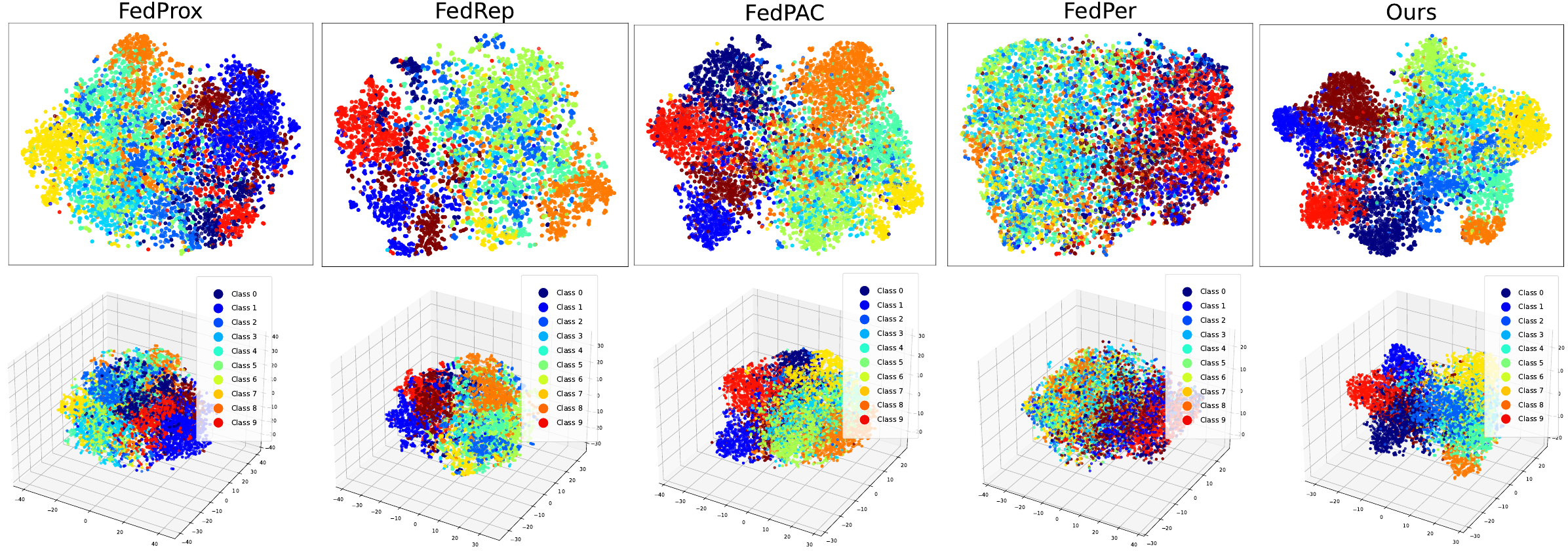}
	\caption{Feature space visualization of different methods on the CIFAR-10 dataset with $\beta=5$.}
	\label{fig7}
\end{figure*}
\subsubsection{Feature Space Analysis with t-SNE}
We conduct a thorough analysis of the feature spaces generated by various methods using t-SNE for dimensionality reduction. 
As shown in Fig.~\ref{fig7}, our method, through the GFA component, effectively maximizes inter-class distances while minimizing intra-class distances. Although FedPAC also leverages feature centroids to construct a well-defined feature space, it requires transmitting these centroids, resulting in higher communication costs and potential privacy risks. 

\begin{figure*}[!t]
	\centering
	\includegraphics[width=7in]{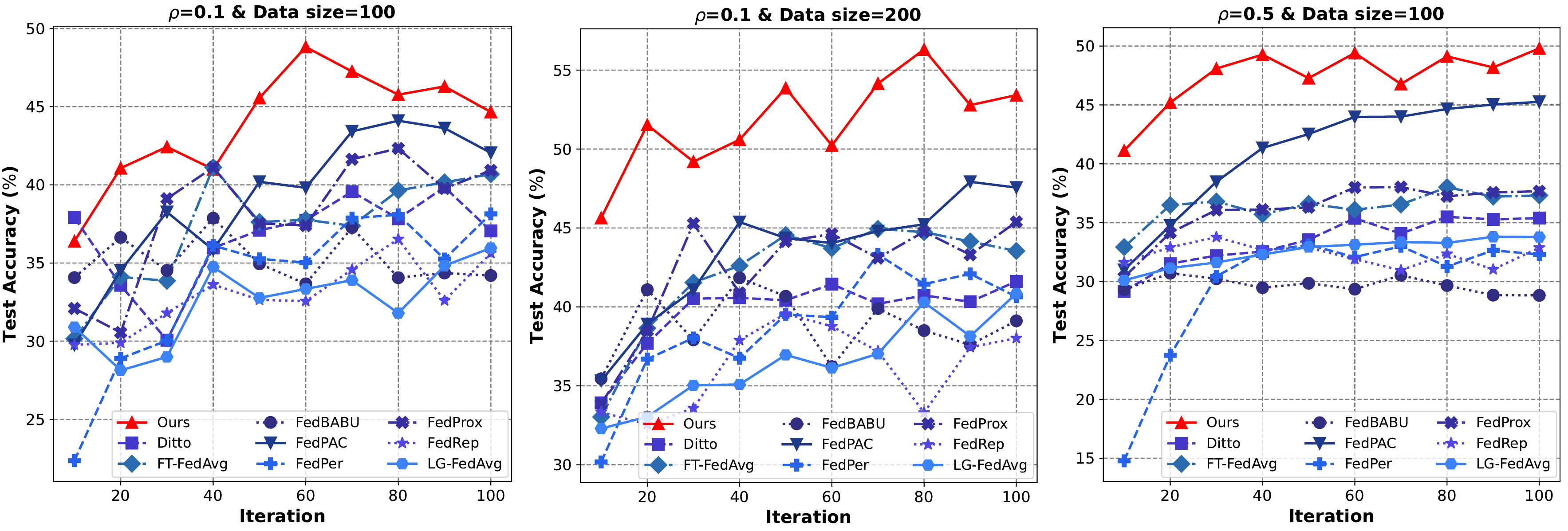}
	\caption{Test accuracy on the CINIC-10 dataset with limited samples and low global aggregation rate. }
	\label{fig8}
    \vspace{-10pt}
\end{figure*}
\subsubsection{Robustness under Data-Scarce Cases}
We evaluate the robustness of our method under conditions where clients possess limited samples and the global aggregation rate is low. Specifically, we conduct experiments on the CINIC-10 dataset with $s=20$, setting the number of samples per client to 100 and 200, and the global aggregation rates $\rho$ to 0.1 and 0.5, respectively. 
As illustrated in Fig.~\ref{fig8}, our method consistently outperforms other approaches across all tested conditions. Notably, the accuracy curves for our method exhibit smoother trajectories, indicating greater stability under resource-constrained environments. This suggests that our approach effectively maintains performance despite reductions in data availability and client participation.

\begin{table}[!t]
    \centering
    \caption{Test accuracy across varying data size.}
    \Large 
    \renewcommand{\arraystretch}{1.2} 
    \setlength{\tabcolsep}{14pt} 
    \begin{minipage}{\columnwidth} 
        \begin{adjustbox}{max width=\columnwidth} 
            \begin{tabular}{lcccc}
                \toprule
                \textbf{Method} & \textbf{200} & \textbf{500} & \textbf{600} & \textbf{1000} \\
                \midrule
                \textit{Local} & \textit{70.84$\pm$0.12} & 
                                \textit{71.48$\pm$0.10} & 
                                \textit{72.44$\pm$0.11} & 
                                \textit{71.96$\pm$0.09} \\
                FT-FedAvg & 67.89$\pm$0.14 & 77.99$\pm$0.11 & 
                            83.48$\pm$0.10 & 86.11$\pm$0.12 \\
                \midrule
                FedProx & 77.06$\pm$0.09 & 80.94$\pm$0.12 & 
                          81.49$\pm$0.11 & 84.98$\pm$0.10 \\
                Ditto & 77.78$\pm$0.12 & 80.30$\pm$0.10 & 
                        83.65$\pm$0.12 & 82.48$\pm$0.11 \\
                \midrule
                FedPer & 65.43$\pm$0.13 & 73.76$\pm$0.11 & 
                         74.80$\pm$0.10 & 76.61$\pm$0.12 \\
                LG-FedAvg & 69.46$\pm$0.11 & 81.72$\pm$0.10 & 
                             81.95$\pm$0.12 & 82.56$\pm$0.11 \\
                FedRep & 72.79$\pm$0.12 & 76.28$\pm$0.11 & 
                         80.23$\pm$0.10 & 83.39$\pm$0.12 \\
                FedPAC & 79.99$\pm$0.11 & 84.06$\pm$0.09 & 
                         84.41$\pm$0.11 & 84.89$\pm$0.10 \\
                FedBABU & 73.31$\pm$0.13 & 80.85$\pm$0.12 & 
                          82.25$\pm$0.10 & 79.39$\pm$0.12 \\
                \midrule
                \textbf{Ours} & \textbf{81.54$\pm$0.10\dag} & 
                \textbf{86.31$\pm$0.08\dag} & 
                \textbf{88.96$\pm$0.09\dag} & 
                \textbf{87.33$\pm$0.07\dag} \\
                
                Gain (\%) & 
                +1.60 & 
                +2.30 & 
                +4.60 & 
                +1.20 \\
                $t$-test / $p$-value & 
                3.0 / 0.04 & 
                3.1 / 0.04 & 
                3.2 / 0.03 & 
                2.9 / 0.04 \\
                \bottomrule
            \end{tabular}
        \end{adjustbox}
    \end{minipage}
    \label{table5}
    \vspace{-15pt}
\end{table}

\subsubsection{Scalability and Stability Analysis across Varying Data Sizes}
We train models on CIFAR-10 datasets of varying scales to evaluate the stability and scalability of each method. The experimental setup includes a Dirichlet distribution parameter of $\beta=0.1$, 100 global communication rounds, and performance evaluation under different client data volumes (200, 500, 600, and 1000 samples).  
As shown in Tab.~\ref{table5}, the results indicate significant performance fluctuations across most methods as data sizes vary, with fine-tuned FedAvg exhibiting the greatest instability. It performs poorly on smaller datasets and stabilizes only when larger datasets are used. 

We further compare these results with those obtained under the practical setting on CIFAR-10 with $\beta=0.1$ and 200 global communication rounds. Most methods fail to stabilize after 100 rounds, resulting in a notable performance gap compared to 200 rounds. For instance, FedPer and FedBABU show poor performance at 100 rounds, while LG-FedAvg achieves better results with fewer rounds, suggesting potential overfitting in the later stages. In contrast, our method demonstrates performance at 100 rounds comparable to that at 200 rounds. 
Overall, our approach consistently outperforms other methods across all data sizes, highlighting its scalability and ability to maintain stable performance across diverse data environments. 

\subsubsection{Stability Evaluation under Varying Aggregation Ratios}
To further assess the stability of our method, we evaluate its performance on the CIFAR-10 dataset with $\beta=0.1$ under three aggregation ratio settings: i) a random ratio within $\rho = [0.1, 1.0]$, ii) a random ratio within $\rho = [0.5, 1.0]$, and iii) a fixed ratio of $\rho = 0.5$. These random ratios simulate real-world FL training scenarios where client dropout is a common occurrence. 
As illustrated in Fig.~\ref{fig5}, most baseline methods struggle to learn effectively due to the instability introduced by varying aggregation ratios. In contrast, our method incorporates the components GLF, GFA, and GPC, which enhance the utilization of global information and balance the local adaptability of feature extractors and classifiers with global generalization. This enables our approach to achieve a more refined trade-off between generalization and personalization. Consequently, our method consistently demonstrates high stability across all aggregation ratios, particularly in the more challenging range of $\rho = [0.1, 1.0]$.
\begin{figure}[!t]
	\centering
	\includegraphics[width=3.0in]{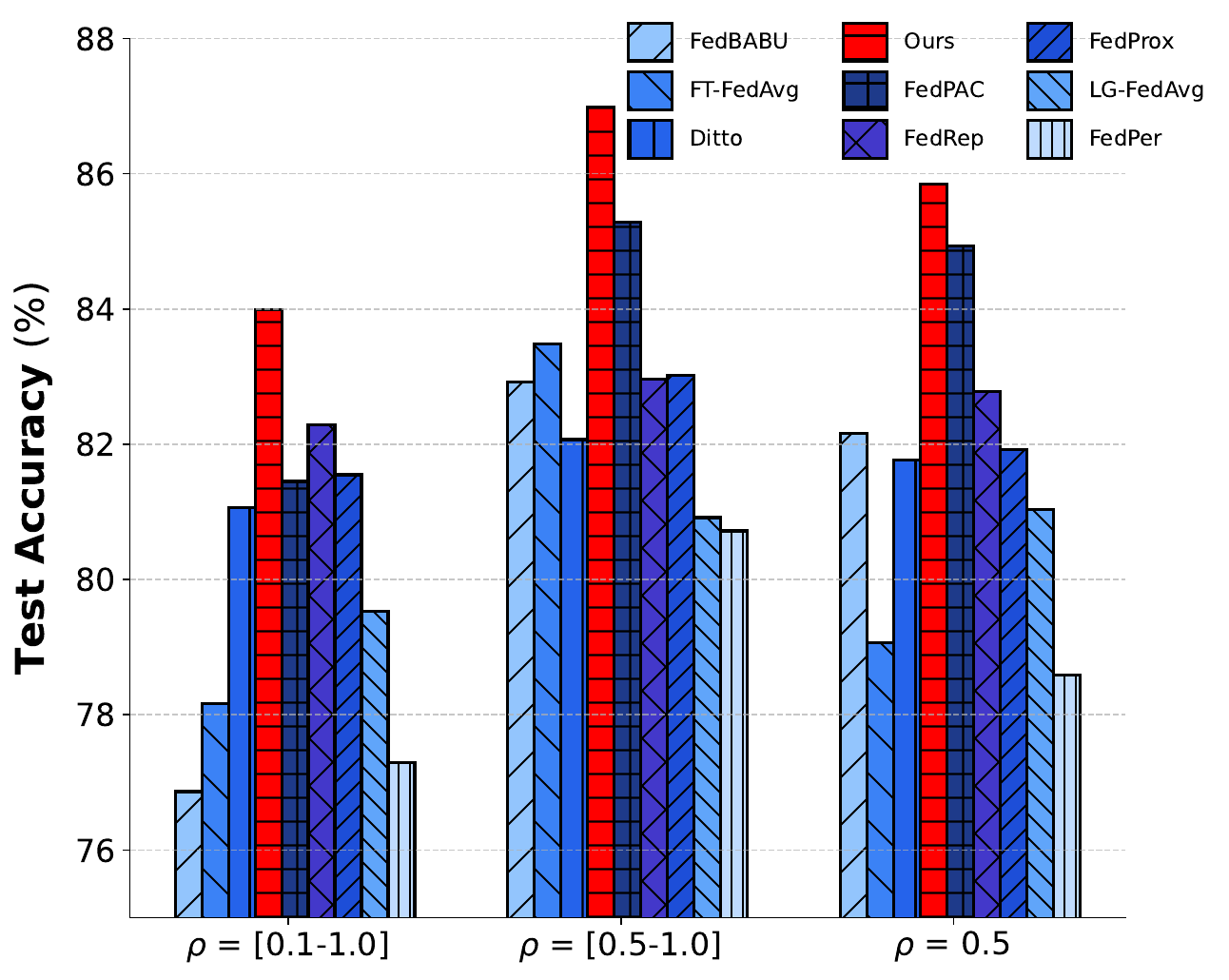}
	\caption{Test accuracy on the CIFAR-10 dataset with two random aggregation ratios ($\rho=[0.1-1.0]$ and $\rho=[0.5-1.0]$) and a fixed aggregation ratio ($\rho=0.5$).}
	\label{fig5}
    \vspace{-15pt}
\end{figure}

\begin{figure}[!t]
	\centering
	\includegraphics[width=3.2in]{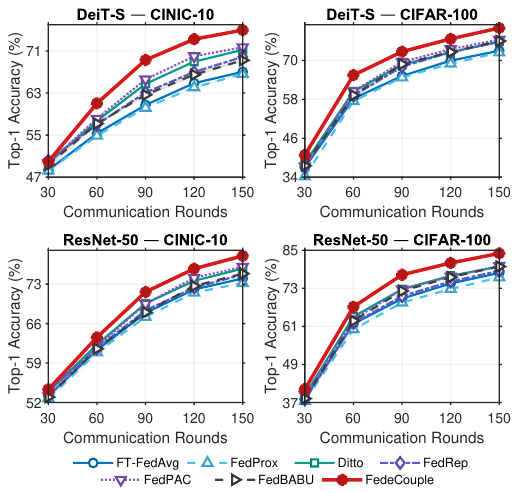}
	\caption{{Accuracy curves on CINIC-10 and CIFAR-100 ($s = 30$) for DeiT-S and ResNet-50 backbones.}}
	\label{trans}
\end{figure}
\subsubsection{{Performance Evaluation on Large-Scale Models}}
{
We further evaluate the scalability and stability of FedeCouple on larger-scale models. 
Experiments are conducted using Vision Transformer (DeiT-S) and ResNet-50 architectures under the same settings as the small-scale CNNs, with evaluation performed on CINIC-10 and CIFAR-100 datasets. 
As shown in Fig.~\ref{trans}, FedeCouple consistently outperforms representative baselines across all four configurations while demonstrating faster convergence and smoother training dynamics. 
These results indicate that even with more complex architectures and higher-dimensional feature spaces, the prototype-derived anchors and center loss constraints effectively promote the synergy between global and personalized knowledge, ensuring robustness across diverse model types and scales.}

\begin{figure}[!t]
	\centering
	\includegraphics[width=3.2in]{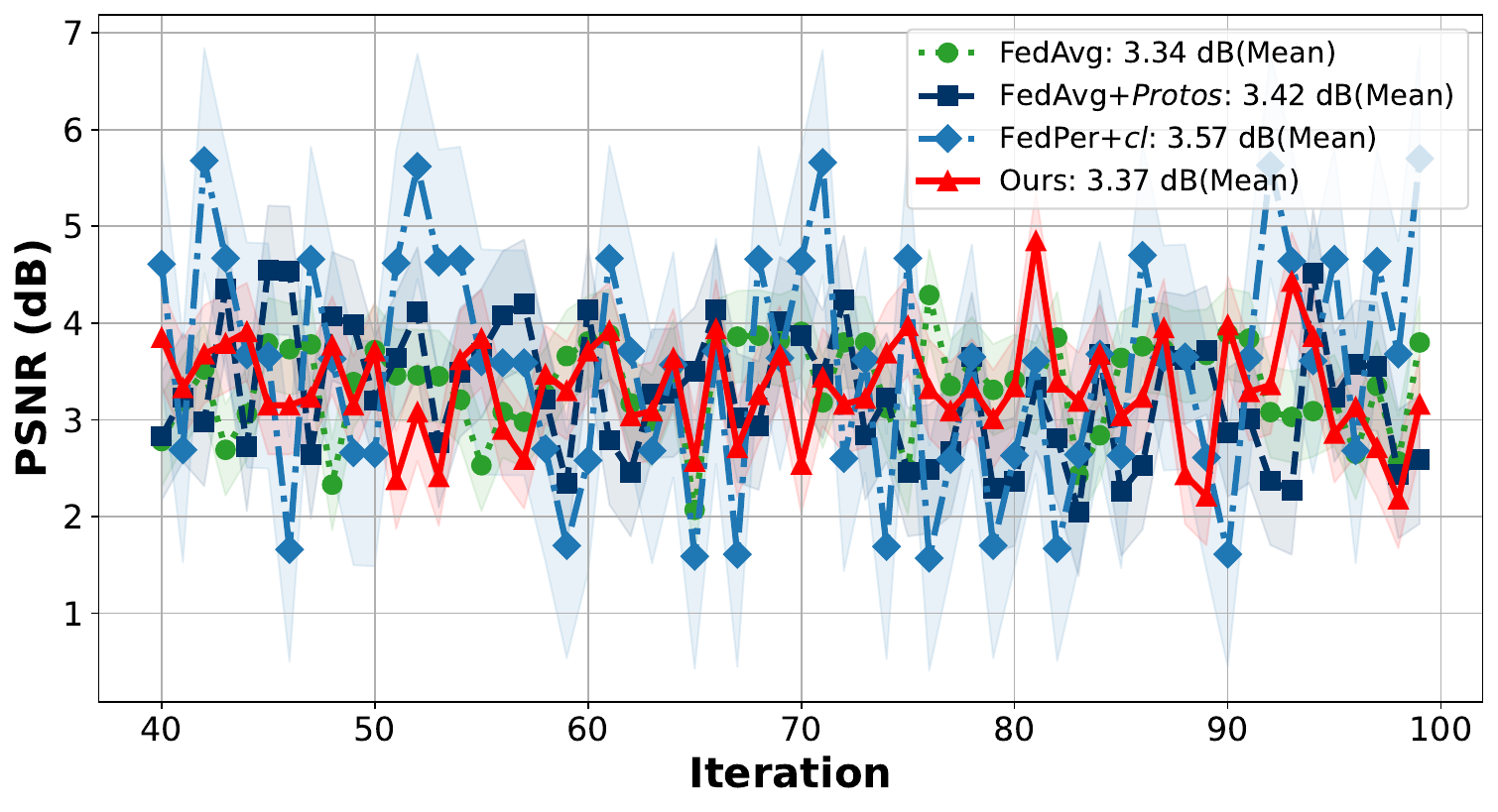}
	\caption{Mean PSNR between inferred and real data under deep gradient leakage attacks for different methods.}
	\label{fig10}
    \vspace{-10pt}
\end{figure}
\subsection{Privacy Leakage Evaluation}\label{privacy}
We additionally evaluate the privacy guarantee of our method against a malicious server attempting to reconstruct original data from client-uploaded gradients. The Peak Signal-to-Noise Ratio (PSNR) \cite{ref63} is employed to quantify the difference between the reconstructed data on the server and the real client data, where a lower PSNR indicates poorer reconstruction quality and thus higher privacy preservation.  

As shown in Fig.~\ref{fig10}, we observe that uploading a personalized classifier in FedPer results in greater privacy leakage compared to FedAvg. This is because personalized information is more closely tied to the client’s local data, potentially enabling the server to infer sensitive content. We also examine scenarios where both gradients and feature centroids are uploaded, as in FedPAC and FedProto. The results demonstrate that such practices expose more client information, making it easier for the server to reconstruct raw data.
In contrast, our method computes feature anchors locally and integrates the GPC component into the classifier, which enhances its inherent generalization capabilities. This approach reduces the transmission of parameter information while simultaneously improving content security. Furthermore, our method avoids uploading class sample counts, thereby further minimizing the risk of privacy leakage.

\begin{table*}[!t]
\centering
\caption{Time (min), rounds, and comm. (MB) to reach 65/80\% on CIFAR-10 ($s{=}20,\beta{=}0.1$), 45/70\% on CINIC-10 ($s{=}50,\beta{=}0.3$), and 44\% on CIFAR-100. Communication = rounds $\times$ 2 $\times$ (415.6\,KB + 0.504$C$\,KB + 2.25$C$\,KB), where $C$ is the number of classes.}
\label{tab:comm_overhead}
\setlength{\tabcolsep}{9pt}
\renewcommand{\arraystretch}{1.4}
\begin{tabular}{lccccc}
\toprule
\textbf{Method} & \textbf{CIFAR-10 ($s=20$)} & \textbf{CIFAR-10 ($\beta=0.1$)} & \textbf{CINIC-10 ($s=50$)} & \textbf{CINIC-10 ($\beta=0.3$)} & \textbf{CIFAR-100 ($\beta=0.01$)} \\
\midrule
FT-FedAvg   & 24.2m / 57r (48.0) & 36.7m / 87r (73.2) & 47.3m / 103r (86.7) & 50.4m / 111r (93.4) & 54.5m / 129r (120.2) \\
FedProx     & 22.1m / 51r (42.9) & 35.0m / 85r (71.5) & 41.6m /  91r (76.6) & 42.7m / 98r (82.4) & 48.7m / 105r (97.9) \\
Ditto       & 27.9m / 42r (35.3) & 39.3m / 77r (64.8) & 49.5m /  85r (71.5) & 55.6m / 93r (78.2) & 52.6m /  89r (82.9) \\
FedPer      & 35.7m / 103r (85.6) & 47.1m / 132r (109.7) & 53.5m / 152r (126.3) & 47.1m / 138r (114.7) & 55.7m / 158r (131.3) \\
FedRep      & 28.5m / 78r (64.8) & 52.6m / 125r (103.9) & 50.1m / 121r (100.6) & 45.4m / 104r (86.4) & 67.5m / 173r (143.8) \\
FedPAC      & 25.9m / 37r (32.8) & 38.2m / 77r (68.2) & 40.8m /  79r (70.0) & 45.4m / 85r (75.3) & 69.9m / 125r (172.8) \\
\textbf{Ours} & \textbf{18.8m / 29r (24.4)} & \textbf{32.2m / 59r (49.6)} & \textbf{35.5m / 63r (53.0)} & \textbf{32.2m / 60r (50.5)} & \textbf{43.3m / 87r (81.1)} \\
\bottomrule
\end{tabular}
\vspace{-5pt}
\end{table*}

\subsection{{Overhead analysis}}
{To evaluate system-level overhead fairly, experiments are conducted in a unified setting on a single NVIDIA 2080Ti GPU. This hardware is intentionally chosen to eliminate computational headroom that might obscure performance differences. 
For each dataset, target accuracy thresholds are predefined; we report both the wall-clock time and the number of communication rounds required to achieve these thresholds. This design minimizes bias from varying final accuracies and better reflects practical deployment constraints. Communication cost is measured as the per-round transmission volume of feature extractors, classifiers, and prototypes. The total cost is defined as the number of rounds multiplied by the upload size per round. 
As summarized in Tab.~\ref{tab:comm_overhead}, our method consistently reaches the target accuracy with fewer rounds and less time, while incurring the lowest communication overhead across all settings. On CIFAR-10 ($s=20$, $\beta=0.1$) and CINIC-10 ($s=50$, $\beta=0.3$), total traffic is reduced by approximately 25\% compared to the most efficient baseline. On the more challenging CIFAR-100 task under severe heterogeneity ($\beta=0.01$), our method again achieves the lowest cost (81.1 MB vs. 82.9 MB). 
These results demonstrate that although local distillation introduces computational cost, its regularization effect reduces overfitting, ultimately yielding a more favorable balance between accuracy and resource efficiency.}

\begin{table}[!t]
	\centering
	\caption{Impacts of Different Components in FedeCouple.}
	\small 
	\renewcommand{\arraystretch}{1.3} 
	\setlength{\tabcolsep}{12pt} 
	\begin{adjustbox}{max width=\textwidth} 
		\begin{tabular}{cccc|c}
			\toprule
			\textbf{GFA} & \textbf{GLF} & \textbf{GPC} & \textbf{DA} & \textbf{Accuracy} \\
			\midrule
			\texttimes \textbar \checkmark & \texttimes \textbar \texttimes & \texttimes \textbar \texttimes & \texttimes \textbar \texttimes & {71.28} \textbar \text{ 77.12} \\
			\midrule
			\texttimes \textbar \texttimes & \checkmark \textbar \texttimes & \texttimes \textbar \checkmark & \texttimes \textbar \texttimes & 75.31 \textbar \text{ 73.40}  \\
			\midrule
			\texttimes \textbar \checkmark & \texttimes \textbar \checkmark & \texttimes \textbar \texttimes & \checkmark \textbar \texttimes & \textbf{77.51} \textbar \text{ 79.68} \\
			\midrule
			\texttimes \textbar \texttimes & \checkmark \textbar \texttimes & \checkmark \textbar \checkmark & \texttimes \textbar \checkmark & 75.70 \textbar  \text{ 78.48}\\
			\midrule
			\checkmark \textbar \checkmark& \texttimes \textbar \texttimes & \checkmark \textbar \texttimes & \texttimes \textbar \checkmark & 77.35 \textbar \textbf{ 80.28} \\
			\midrule
			\texttimes \textbar \checkmark & \checkmark \textbar \checkmark & \texttimes \textbar \checkmark & \checkmark \textbar \texttimes & 78.13 \textbar \text{ 78.36} \\
			\midrule
			\checkmark \textbar \checkmark & \checkmark \textbar \texttimes & \texttimes \textbar \checkmark & \checkmark \textbar \checkmark & 81.58 \textbar \textbf{ 81.81} \\
			\midrule
			\texttimes \textbar \checkmark & \checkmark \textbar \checkmark & \checkmark \textbar \checkmark & \checkmark \textbar \checkmark & 76.38 \textbar \textbf{ 83.42} \\
			\bottomrule
		\end{tabular}
	\end{adjustbox}
	\label{table6}
\end{table}
\subsection{Ablation-Study}
\subsubsection{Components' Effectiveness Evaluation} We evaluate the effectiveness of the GLF, GFA, GPC, and data augmentation (DA) components on the CIFAR-10 dataset with $s=20$. Detailed results are presented in Tab.~\ref{table6}, where "\texttimes" indicates that the component is not used, and "\checkmark" indicates that the component is used during training. 

The analysis demonstrates that each component contributes to the overall performance. Specifically, DA proves to be the most effective individual component. When two components are combined, the pairing of GFA and DA yields the best performance. When three components are incorporated, the combination of GFA, GPC, and DA achieves the highest effectiveness, highlighting the synergistic impact of these components in enhancing the performance of our method. 


\setlength{\tabcolsep}{2pt}  
\renewcommand{\arraystretch}{1.2}  
\begin{table}[t]
\centering
\caption{Fine-grained ablation study of GLF, GFA, and GPC, with results reported as classification accuracy (\%).}
\label{tab:fine_ablation}
\begin{tabular}{@{}>{\centering\arraybackslash}m{0.8cm} 
                >{\centering\arraybackslash}m{3.1cm} 
                >{\centering\arraybackslash}m{4.7cm}@{}} 
\toprule
\textbf{Comp.} & \textbf{Sub-setting} & \textbf{Accuracy} \\
\midrule
\multirow{2}{*}{\textbf{GLF}} 
 & Global extractor update        
 & w/o: 71.5 \textbar\; w/: \textbf{74.3} \\
 & Frozen classifier joint train  
 & w/o: 71.2 \textbar\; w/: \textbf{73.5} \\
\midrule
\textbf{GFA} & Anchor coverage               
 & 0\%: 71.9 \textbar\; 50\%: 75.7 \textbar\; 100\%: \textbf{\underline{77.5}} \\
\midrule
\textbf{GPC} & KD variants\textsuperscript{*}  
 & w/o: 71.0 \textbar\; Fixed: 72.2 \textbar\; Adapted: \textbf{73.9} \\
\bottomrule
\end{tabular}

\vspace{2pt}
\raggedright\footnotesize \textsuperscript{*}KD = Knowledge Distillation. “Fixed” means teacher is frozen, “Adapted” means teacher is locally fine-tuned.
\vspace{-3pt}
\end{table}

\subsubsection{{Fine-grained Sub-module Ablation}}
{We conduct a fine-grained ablation study to evaluate the individual contributions of each design component. 
For GLF, we examine: i) whether the local feature extractor benefits from updates guided by the global extractor, and ii) whether additional improvements arise from joint training with a frozen global classifier.
For GFA, we compare three anchor coverage settings: no anchors, partial coverage, and full coverage. 
For GPC, we assess: i) the effect of distillation from the global classifier to the local classifier, and ii) whether the teacher remains fixed or is updated via local adaptation. 
Results in Tab.~\ref{tab:fine_ablation} show that each sub-module independently enhances performance, consistent with the analyses in Section~\ref{me}.
Furthermore, the sub-modules are complementary, and their combination yields further improvements. Anchors generated from local data using the global extractor achieve the best results, underscoring the importance of a feature space with high inter-class separability and intra-class compactness for downstream tasks.}

\begin{table}[!t]
	\centering
	\caption{Impact of Hyperparameter $\lambda$.}
	\small 
	\renewcommand{\arraystretch}{1.2} 
	\setlength{\tabcolsep}{9pt} 
	\begin{tabular}{cccccc}
		\toprule
		\textbf{Dataset} & 0.2 & 0.4 & 0.6 & \textbf{0.8} & 1.0 \\
		\midrule
		\textbf{CINIC-10} & 59.35 & 59.66 & 60.13 & \textbf{61.55}  & 59.65 \\
		\bottomrule
		\textbf{CIFAR-10} & 78.66 & 80.51 & \textbf{82.12} & 81.46& 79.55 \\
		\bottomrule
		\textbf{FMNIST} & 91.95 & 92.06 & 92.11 & \textbf{92.38}  & 91.89 \\
		\bottomrule
	\end{tabular}
	\label{table8}
\end{table}

\begin{table}[!t]
	\centering
	\caption{Impact of Hyperparamete $\mu$.}
	\small 
	\renewcommand{\arraystretch}{1.2} 
	\setlength{\tabcolsep}{9pt} 
	\begin{tabular}{cccccc}
		\toprule
		\textbf{Dataset} & 0.2 & 1.0 & \textbf{2.0} & {3.0} & 5.0 \\
		\midrule
		\textbf{CINIC-10} & 58.21 & 59.11 & \textbf{62.03} & 61.35& 61.10 \\
		\bottomrule
		\textbf{CIFAR-10} & 78.21 & 79.95 & 81.55 & \textbf{82.21}  & 81.06 \\
		\bottomrule
		\textbf{FMNIST} & 91.64 & 91.86 & \textbf{92.35} & 92.17& 91.13 \\
		\bottomrule
	\end{tabular}
	\label{table9}
\end{table}
\begin{table}[!t]
	\centering
	\caption{Impact of Local Epochs.}
	\scriptsize 
	\renewcommand{\arraystretch}{1.7} 
	\setlength{\tabcolsep}{5.0pt} 
	\begin{tabular}{ccccccccc}
		\toprule
		\textbf{Dataset} & 1 & 3 & 4 & \textbf{5} & 6 & 7 & 8 & 10 \\
		\midrule
		\textbf{CINIC-10} & 56.96 & 58.04 & 59.63 & \textbf{61.11} & 59.20 & 59.11 & 56.56 & 58.76 \\
		\bottomrule
		\textbf{CIFAR-10} & 76.76 & 79.10 & 79.46 & 81.78& \textbf{82.15} & 80.86 & 79.54 & 79.86 \\
		\bottomrule
		\textbf{FMNIST} & 90.78 & 91.93 & 91.89 & \textbf{92.35} & 92.06 & 91.95 & 92.07 & 91.87 \\
		\bottomrule
	\end{tabular}
	\label{table7}
\end{table}
\subsubsection{Impacts of Hyperparameters and Local-Epochs}
We evaluate the hyperparameters within the loss function under the setting $s=20$. As shown in Tabs.~\ref{table8} and~\ref{table9}, the regularization coefficients $\lambda$ and $\mu$ significantly influence performance. 
{We further analyze the roles of the two key hyperparameters. Specifically, $\lambda$ controls the integration of global knowledge into the local classifier via distillation, while $\mu$ regulates the strength of the global anchor constraints on the feature extractor. 
Together, they balance the degree of global knowledge transfer. Excessive transfer diminishes local adaptability, whereas insufficient transfer hinders generalization. 
Due to its deeper and more parameter-rich architecture, the feature extractor benefits from stronger global regularization; thus, a larger $\mu$ (e.g., 2.0) yields improved stability and performance. In contrast, the classifier is more lightweight and sensitive to global guidance, making a smaller $\lambda$ (e.g., 0.8) effective for balancing generalization with personalization.}
Additionally, Tab.~\ref{table7} demonstrates that both excessive and insufficient local epochs negatively impact performance, with five epochs yielding the best results.

\section{Conclusion}\label{cf}
This paper proposes a novel PFL method that effectively addresses the challenge of data heterogeneity while controlling transmission costs in mobile networks and enhancing privacy protection. By focusing on the local adaptability of the global feature extractor and the global consensus learning of the local classifier, our method achieves an appropriate balance between generalization and personalization. Extensive experiments and theoretical analyses demonstrate that the proposed method delivers superior performance in terms of effectiveness, stability, and scalability. 

Future work will focus on optimizing the iteration strategy to enhance the synergy between global and local feature information. Additionally, we aim to explore more refined model decoupling approaches, moving beyond the traditional division between feature extractor and classifier. Techniques such as hypernetworks or semantic segmentation will be investigated to enable dynamic model partitioning based on data sensitivity or other relevant factors. 

\section*{Acknowledgment}
\small{This work was supported in part by the Taishan Scholars Program under Grants tsqn202211203 and tsqn202408239, in part by the NSFC under Grant 62402256, in part by the Shandong Provincial Nature Science Foundation of China under Grant ZR2024MF100, in part by the Pilot Project for the Integration of Science, Education and Industry of Qilu University of Technology (Shandong Academy of Sciences) under Grant 2025ZDZX01, and in part by Open Research Project of the State Key Laboratory of Industrial Control Technology, Zhejiang University, China under Grant ICT2025B13.}
\bibliographystyle{IEEEtran}
\bibliography{references}

\vspace{-30pt}
\begin{IEEEbiographynophoto}{Ming Yang (Member, IEEE)}
received the B.S. and M.S. degrees from the School of Information Science and Engineering, Shandong University, Jinan, China, in 2004 and 2007, respectively, and the Ph.D. degree from the School of Electronic Engineering, Beijing University of Posts and Telecommunications, Beijing, China, in 2010. He is currently a Professor at the Shandong Computer Science Center, Qilu University of Technology (Shandong Academy of Sciences), Jinan. His research interests include AI security, data privacy, and network security.
\end{IEEEbiographynophoto}
\begin{IEEEbiographynophoto}{Dongrun Li}
received the B.E. degree in Internet of Things Engineering from Qilu University of Technology (Shandong Academy of Sciences), Jinan, China, in 2023. Currently, he is working toward the MS degree with the Department of Computer Science and Technology, Qilu University of Technology (Shandong Academy of Sciences). His research interests include federated learning, large and small models collaboration, and representation learning.    
\end{IEEEbiographynophoto}
\begin{IEEEbiographynophoto}{Xin Wang (Member, IEEE)}
received the B.E. degree in Electrical Engineering and Automation from China University of Mining and Technology, Xuzhou, China, in 2015, and the Ph.D. degree in Control Science and Engineering from Zhejiang University, Hangzhou, China, in 2020. He was a visiting scholar in the Department of Computer Science, Tokyo Institute of Technology, Yokohama, Japan, from 2018 to 2019. He is currently a Professor at the Shandong Computer Science Center, Qilu University of Technology (Shandong Academy of Sciences), Jinan, China.  His research interests include distributed artificial intelligence, federated learning, and data security.
\end{IEEEbiographynophoto}
\vspace{-10pt}
\begin{IEEEbiographynophoto}{Feng Li (Member, IEEE)}
received his PhD degree from Nanyang Technological University, Singapore, in 2015. He obtained the BS and MS degrees from Shandong Normal University, China, in 2007, and Shandong University, China, in 2010, respectively. From 2014 to 2015, he worked as a research fellow in National University of Singapore, Singapore. He then joined School of Computer Science and Technology, Shandong University, China, where he is currently a professor. His research interests include distributed algorithms and systems, wireless networking, mobile computing, and Internet of Things. 
\end{IEEEbiographynophoto}
\vspace{-10pt}
\begin{IEEEbiographynophoto}{Lisheng Fan}
received the bachelor's degree from Fudan University in 2002 and the master’s degree from Tsinghua University, China, in 2005, both from the Department of Electronic Engineering, and the Ph.D. degree from the Department of Communications and Integrated Systems, Tokyo Institute of Technology, Japan, in 2008. He is currently a Professor with the School of Computer Science, Guangzhou University. His research interests include wireless cooperative communications, physical-layer secure communications, intelligent communications, and system performance evaluation.
\end{IEEEbiographynophoto}
\vspace{-10pt}
\begin{IEEEbiographynophoto}{Chunxiao Wang} received the bachelor’s degree from Liaocheng University, Liaocheng, China, in 2002, and the master’s degree from Shandong University, Jinan, China, in 2006. Since 2006, she has been with the Shandong Computer Science Center, Qilu University of Technology (Shandong Academy of Sciences), Jinan, China, where she is currently an Associate Professor. Her research interests include computing and storage resource orchestration, big data analysis, and software engineering.
\end{IEEEbiographynophoto}
\vspace{-10pt}
\begin{IEEEbiographynophoto}{Xiaoming Wu}
received the M.Eng. degree in computer science and technology from Shandong University, Jinan, China, in 2006, and the Ph.D. degree in Software Engineering from Shandong University of Science and Technology in 2017. Since 2006, he has been with the Shandong Computer Science Center, where he is currently a full professor. He also serves as the director of the Faculty of Computer Science and Technology at Qilu University of Technology (Shandong Academy of Sciences), China. His research interests include cyber security, industrial Internet, data security, and privacy protection.
\end{IEEEbiographynophoto}
\vspace{-10pt}
\begin{IEEEbiographynophoto}{Peng Cheng (Member, IEEE)}
received the B.E. degree in Automation, and the Ph.D. degree in Control Science and Engineering from Zhejiang University, Hangzhou, China, in 2004 and 2009, respectively. He was a Research Fellow with Information System Technology and Design Pillar, Singapore University of Technology and Design, Singapore, from 2012 to 2013. He is currently a Changjiang Scholars Chair Professor and the Executive Dean of the College of Control Science and Engineering, Zhejiang University. His research interests include control system security and cyber-physical systems. 
\end{IEEEbiographynophoto}

\end{document}


\vspace{11pt}
\clearpage
\twocolumn 
\onecolumn
\appendices
\pagenumbering{arabic} 
\setcounter{page}{1}

\section{Convergence Analysis for FedeCouple}\label{appa}
Drawing inspiration from the work of FedProto and FedFA, we conduct a convergence analysis for our proposed method. The loss function of local model is defined as follows:
    $$\mathcal{L}_{i}^{t}=\ell_{ce}(\theta_{i}^{t}, \phi_{i}^{t}; \gamma_{i}^t)+\lambda\ell_{kl}(p_{T_i}^t,p_{S_i}^t)+\mu\Omega_i^t.$$

Before presenting the theoretical results, we first establish the following assumptions:

\begin{assumption}\label{aas1} 
({Lipschitz Smoothness}) For $\forall t_1,t_2>0, \text{and } i\in\{1,2,...,N\}$, the objective function is $L_1$-smooth:
    $$
    \mathcal{L}_{i}^{t_{1}}-\mathcal{L}_{i}^{t_{2}}
    \leq
    \langle\nabla \mathcal{L}_{i}^{t_{2}},\,
    \omega_{i}^{t_{1}}-\omega_{i}^{t_{2}}\rangle
    +\frac{L_{1}}{2}\left\|\omega_{i}^{t_{1}}-\omega_{i}^{t_{2}}\right\|_{2}^{2}.
    $$
\end{assumption}

\noindent\textbf{{Remark.}}
{Although Assumption~\ref{aas1} requires differentiability, the analysis can be extended to piecewise-linear activations such as ReLU, which are non-differentiable at zero. 
In such cases, $\nabla \mathcal{L}$ can be interpreted as a generalized (Clarke) subgradient $g \in \partial \mathcal{L}(\omega)$, where the generalized subgradient refers to selecting a value within the range defined by the left and right derivatives at a non-differentiable point. 
For example, for the ReLU function $f(x)=\max(0,x)$, the derivative is $1$ when $x>0$, $0$ when $x<0$, and its generalized subgradient set at $x=0$ is the interval $[0,1]$. 
Under this interpretation, the descent lemma holds in the weakly-smooth form:
$$
\mathcal{L}^{t_1}_i-\mathcal{L}^{t_2}_i\le  \langle g,\, \omega^{t_1}_i-\omega^{t_2}_i \rangle + \tfrac{L_1}{2}\|\omega^{t_1}_i-\omega^{t_2}_i\|^2_2, 
\quad g \in \partial \mathcal{L}(\omega_i^{t_2}).
$$
Therefore, in the subsequent analysis, $\nabla \mathcal{L}$ in the non-smooth case can be understood as a generalized gradient, so the convergence results naturally extend to ReLU-based networks.}

\begin{assumption}\label{aas2}
(Unbiased Gradient and Bounded Variance) The expectation of the gradient obtained from a single sampling is equal to the gradient obtained from full sampling:
    $$\mathbb{E}_{\gamma_{i}^{t}\sim \mathcal{D}_{i}}\left[gr_{i}^{t}\right]=\nabla\mathcal{L}\left(\omega_{i}^{t}\right)=\nabla\mathcal{L}_{i}^{t},$$
its variance is bounded by $\sigma^2$:
    $$\mathbb{E}\left[\left\|gr_i^t-\nabla\:\mathcal{L}_i^t\right\|_2^2\right]\leq\sigma^2.$$
\end{assumption}
The above assumption also applies to the training of the global classifier.

\begin{assumption}\label{aas3} (Bounded Expectation of Euclidean Norm of Single-sample Gradients) The expectation of $\mathbb{L}_2$-norm of single sample gradient is bounded by
    $$\mathbb{E} \left[\left\|gr_i^t\right\|_2\right] \leq G.$$
\end{assumption}
The above assumption also applies to the training of the global classifier.

\begin{assumption}\label{aas4} 
(Lipschitz Continuity) The loss function of feature extractor is $L_2$-Lipschitz continuous:
    $$\left\| f_{i}\left(\theta_{i}^{t_{1}}\right)-f_{i}(\theta_{i}^{t_{2}})\right\|_{2}^{2}\leq L_{2}\left\|\theta_{i}^{t_{1}}-\theta_{i}^{t_{2}}\right\|_{2}^{2}.$$
\end{assumption}
The above assumption also applies to the updating of the local and global classifiers.

The total number of training iterations is denoted by $t$, with $E$ local epochs between each global communication round. Therefore, after the $k$-th global round, the total number of training iterations is given by $t = kE$. Additionally, We assume that the process of the aggregating global information, sending it to the local client, and performing local updates consumes half a segment of time. As a result, the local training progress between two consecutive global communication rounds is represented by $e \in \{1/2, 1, 2, \dots, E\}$. Furthermore, we divide the interval $0-1/2$ into two subintervals: $0-1/4$, which corresponds to local model update, and $1/4-1/2$, which corresponds to global information update.

We begin by analyzing the variation of the loss function throughout the training process.
\begin{lemma}\label{ale1}
If the local model is trained via stochastic gradient descent (SGD), we have
\begin{equation*}
\mathbb{E}\left(\mathcal{L}_i^{(t+1)E}\right) -\mathcal{L}_i^{tE+1/2} \leq\ L_1 \eta^2 \left(E\sigma^2 + \sum_{e=1/2}^{E-1} \|\nabla \mathcal{L}_i^{tE+e}\|_2^2 \right)-\eta \sum_{e=1/2}^{E-1} \|\nabla \mathcal{L}_i^{tE+e}\|_2^2.
\end{equation*}
\end{lemma}

\begin{proof}
According to the Assumption~\ref{aas1} and the training rule $\omega_i^{t+1}=\omega_i^{t}-\eta gr_i^{t}$, we obtain
\begin{equation*}
	\begin{split}
		\mathcal{L}_{i}^{tE+1}-\mathcal{L}_i^{tE+1/2}&{\operatorname*{\leq}}\langle\nabla\mathcal{L}_i^{tE+1/2},\left(\omega_i^{tE+1}-\omega_i^{tE+1/2}\right)\rangle+\frac{L_1}2\left\|\omega_i^{tE+1}-\omega_i^{tE+1/2}\right\|_2^2\\&=-\eta\langle\nabla\mathcal{L}_i^{tE+1/2},gr_i^{tE+1/2}\rangle+\frac{L_1}2\left\|\eta gr_i^{tE+1/2}\right\|_2^2.
	\end{split}
\end{equation*}
Taking the expectation on both sides of the above equation with respect to the random variable $\gamma$, and applying Assumption~\ref{aas2} along with the inequality $\|a+b\|_2^2\leq2\|a\|_2^2+2\|b\|_2^2$, we have
\begin{equation*}
    \begin{split}
    \mathbb{E}\left(\mathcal{L}_{i}^{tE+1}\right)-\mathcal{L}_i^{tE+1/2}\leq& -\eta \mathbb{E} \Big[\langle \nabla \mathcal{L}_i^{tE+1/2}, gr_i^{tE+1/2} \rangle \Big]+ \frac{L_1 \eta^2}{2} \mathbb{E} \Big[\left\| \left(gr_i^{tE+1/2}-\nabla \mathcal{L}_i^{tE+1/2}\right)+\nabla \mathcal{L}_i^{tE+1/2} \right\|_2^2 \Big]
    \\&{\leq} - \left( \eta - L_1 \eta^2 \right) \left\|\nabla \mathcal{L}_i^{tE+1/2} \right\|_2^2 
		+ L_1 \eta^2 \sigma^2.
    \end{split}
\end{equation*}
Further, by telescoping of $E$ epochs and simplifying the equation, we obtain 
\begin{equation*}
    \mathbb{E}\left(\mathcal{L}_i^{(t+1)E}\right)- \mathcal{L}_i^{tE+1/2} \leq\ L_1 \eta^2 \left(E\sigma^2 + \sum_{e=1/2}^{E-1} \left\|\nabla \mathcal{L}_i^{tE+e}\right\|_2^2 \right)-\eta \sum_{e=1/2}^{E-1} \left\|\nabla \mathcal{L}_i^{tE+e}\right\|_2^2.
\end{equation*}
\end{proof}

Next, we analyze loss function variation during the local model update process.
\begin{lemma}\label{ale2}
The variation of the loss function during the local model update process is bounded by
\begin{equation*}
    \mathbb{E}(\mathcal{L}_i^{(t+1)E+1/4})-\mathcal{L}_i^{(t+1)E}\leq -\eta \sum_{i=1}^{N} \iota_{i}^{tE} \sum_{e=1/2}^{E-1} \nabla \mathcal{L}_{i}^{tE+e} + \eta \sum_{e=1/2}^{E-1} \nabla \mathcal{L}_{i}^{tE+e} + 2 \lambda \eta^{2} L_{2}^2  E G^{2} |\mathcal{D}_{i}| .
\end{equation*}
\end{lemma}

\begin{proof}
We only update the local feature extractor using the latest global feature extractor. By analyzing the loss function and applying the inequality 
$\|C-A\|_2^2-\|C-B\|_2^2\leq\|A-B\|_2^2$, we derive
\begin{flalign*}
	\mathcal{L}_i^{(t+1)E+1/4} - \mathcal{L}_i^{(t+1)E} &&
\end{flalign*}
\begin{equation*}
    =\left(\omega_i^{(t+1)E+1/4}-\omega_i^{(t+1)E}\right)+\left(\lambda\sum_{l=1}^{\mathcal{D}_i}\left\|f_i(\theta_i^{(t+1)E+1/4};x_l)-An_{y_l}^{tE+1/2}\right\|_2^2-\lambda\sum_{l=1}^{\mathcal{D}_i}\left\|f_i(\theta_i^{(t+1)E};x_l)-An_{y_l}^{tE+1/2}\right\|_2^2\right)
\end{equation*}
\begin{equation*}
    \leq\theta_i^{(t+1)E+1/4}-\theta_i^{(t+1)E}+\lambda\sum_{l=1}^{\mathcal{D}_i}\left\| f_i\left(\theta_i^{(t+1)E+1/4}\right)-f_i\left(\theta_i^{(t+1)E}\right)\right\|_2^2.
\end{equation*}
According to Assumption~\ref{aas4}, we have
\begin{equation*}
    \mathcal{L}_i^{(t+1)E+1/4} - \mathcal{L}_i^{(t+1)E}\leq\sum_{i=1}^N\iota_i^{(t+1)E}\theta_i^{(t+1)E}-\theta_i^{(t+1)E}+\lambda L_2^2|D_i|\left\|\sum_{i=1}^N\iota_i^{(t+1)E}\theta_i^{(t+1)E}-\theta_i^{(t+1)E}\right\|_2^2.
\end{equation*}
Considering that both $\theta$ and $\phi$ constitute parts of $\omega$, and taking into account the aggregation strategy, we further scale the expression, thereby obtaining
\begin{flalign*}
	\mathcal{L}_i^{(t+1)E+1/4} - \mathcal{L}_i^{(t+1)E} &&
\end{flalign*}
\begin{equation*}
    \leq\sum_{i=1}^N\left(\iota_i^{(t+1)E}\omega_i^{(t+1)E}-\iota_i^{tE}\omega_i^{tE}\right)+\sum_{e=1/2}^{E-1}\eta gr_i^{tE+e}+\lambda L_2^2|\mathcal{D}_i|\left\|\sum_{i=1}^N\left(\iota_i^{(t+1)E}\omega_i^{(t+1)E}-\iota_i^{tE}\omega_i^{tE}\right)+\sum_{e=1/2}^{E-1}\eta gr_i^{tE+e}\right\|_2^2.
\end{equation*}
According to the experiment and analysis conducted in Appendix \ref{agg}, the aggregation weight variation between adjacent communication rounds is minimal. Therefore, we have
\begin{equation*}
    \mathcal{L}_i^{(t+1)E+1/4} - \mathcal{L}_i^{(t+1)E}\leq\sum_{i=1}^N\iota_i^{tE}(\omega_i^{(t+1)E}-\omega_i^{tE})+\sum_{e=1/2}^{E-1}\eta gr_i^{tE+e}+\lambda L_2^2|\mathcal{D}_i|\left\|\sum_{i=1}^N\iota_i^{tE}\left(\omega_i^{(t+1)E}-\omega_i^{tE}\right)+\sum_{e=1/2}^{E-1}\eta gr_i^{tE+e}\right\|_2^2.
\end{equation*}
Next, by constructing the zero term, we further scale the expression to obtain the following result:
\begin{equation*}
	\begin{split}
		\mathcal{L}_i^{(t+1)E+1/4} - \mathcal{L}_i^{(t+1)E}\leq &-\sum_{i=1}^N\iota_i^{tE}\sum_{e=1/2}^{E-1}\eta gr_i^{tE+e}+\sum_{i=1}^N\iota_i^{tE}\left(\sum_{i=1}^N\iota_i^{tE}\omega_i^{tE}-\omega_i^{tE}\right)+\sum_{e=1/2}^{E-1}\eta gr_i^{tE+e} \\
		& +\lambda L_2^2|\mathcal{D}_i|\left\|-\sum_{i=1}^N\iota_i^{tE}\sum_{e=1/2}^{E-1}\eta gr_i^{tE+e}+\sum_{i=1}^N\iota_i^{tE}\left(\sum_{i=1}^N\iota_i^{tE}\omega_i^{tE}-\omega_i^{tE}\right)+\sum_{e=1/2}^{E-1}\eta gr_i^{tE+e}\right\|_2^2
	\end{split}
\end{equation*}
\begin{equation*}
	\begin{split}
		\leq & -\sum_{i=1}^{N} \iota_{i}^{tE} \sum_{e=1/2}^{E-1} \eta g r_{i}^{tE+e} + \sum_{e=1/2}^{E-1} \eta g r_{i}^{tE+e}  + \lambda L_{2}^2 |\mathcal{D}_{i}| \left\|-\sum_{i=1}^{N} \iota_{i}^{tE} \sum_{e=1/2}^{E-1} \eta g r_{i}^{tE+e} \right. \left. + \sum_{e=1/2}^{E-1} \eta g r_{i}^{tE+e} \right\|_{2}^{2}.
	\end{split}
\end{equation*}
Taking the expectation over $\gamma$ and applying Assumption~\ref{aas3}, we obtain the following bound for the local model update process:
\begin{equation*}
	\begin{split}
		\mathbb{E}(\mathcal{L}_i^{(t+1)E+1/4})-\mathcal{L}_i^{(t+1)E}\leq & -\eta \sum_{i=1}^{N} \iota_{i}^{tE} \sum_{e=1/2}^{E-1} \nabla \mathcal{L}_{i}^{tE+e} + \eta \sum_{e=1/2}^{E-1} \nabla \mathcal{L}_{i}^{tE+e} + 2 \lambda \eta^{2} L_{2}^2 E G^{2} |\mathcal{D}_{i}|.
	\end{split}
\end{equation*}
\end{proof}

Furthermore, we analyze the loss function variation during the global information update process.
\begin{lemma}\label{ale3}
The variation of the loss function during the global information update process is bounded by
\begin{equation*}
    \mathbb{E}\left(\mathcal{L}_{i}^{(t+1)E+1/2}\right) - \mathbb{E}\left(\mathcal{L}_{i}^{(t+1)E+1/4}\right) \leq \lambda \eta^2 L_2^2EG^2|\mathcal{D}_i|+\frac{2\mu \eta L_2 E G}\tau.
\end{equation*}
\end{lemma}

\begin{proof}
We impose global information constraints on both the feature extractor and the classifier, 
and we analyze them separately. We first focus on the GFA component, simplifying the GPC component to $\mu \ell_{kl}$. By applying the inequality 
$\left\|C - A\right\|_2^2 - \left\|C - B\right\|_2^2 \leq \left\|A - B\right\|_2^2$, we obtain
\begin{equation*}
	\begin{aligned}
		&\mathcal{L}_i^{(t+1)E+1/2}-\mathcal{L}_i^{(t+1)E+1/4} \\ &=\lambda\sum_{l=1}^{\mathcal{D}_i}\left\|f_i\left(\theta_i^{(t+1)E+1/4};x_l\right)-An_{y_l}^{(t+1)E+1/2}\right\|_2^2-\lambda\sum_{l=1}^{\mathcal{D}_i}\left\|f_i\left(\theta_i^{(t+1)E+1/4};x_l\right)-An_{y_l}^{tE+1/2}\right\|_2^2+\mu \ell_{kl}\\ &\leq\lambda\sum_{l=1}^{\mathcal{D}_i}\left\|An_{y_l}^{(t+1)E+1/2}-An_{y_l}^{tE+1/2}\right\|_2^2+\mu \ell_{kl}.
	\end{aligned}
\end{equation*}
Based on the feature anchor expression $An_k^t=\sum_{x_{i,k}\in\mathcal{D}_{i,k}}\frac{f(\theta^t;x_{i,k})}{|\mathcal{D}_{i,k}|}$, we have
\begin{equation*}
    \mathcal{L}_i^{(t+1)E+1/2}-\mathcal{L}_i^{(t+1)E+1/4}\leq\frac\lambda{|\mathcal{D}_i|^2}\sum_{l=1}^{\mathcal{D}_i}\left\|\sum_{l=1}^{\mathcal{D}_i}\left[f(\theta^{(t+1)E+1/2};x_l)-f(\theta^{tE+1/2};x_l)\right]\right\|_2^2+\mu \ell_{kl}.
\end{equation*}
Following a similar analysis as in Lemma~\ref{ale2} (taking the expectation over $\gamma$ for the GLF component), we obtain
\begin{equation*}
	\begin{aligned}
	\mathcal{L}_i^{(t+1)E+1/2}-\mathcal{L}_i^{(t+1)E+1/4}	\leq\lambda\left|\mathcal{D}_i\right|\left\|L_2\sum_{i=1}^{N}\iota_{i}^{tE}\left(\omega_{i}^{(t+1)E}-\omega_{i}^{tE}\right)\right\|_{2}^{2}+\mu \ell_{kl},\\
        \mathbb{E}\left(\mathcal{L}_{i}^{(t+1)E+1/2}\right) - \mathbb{E}\left(\mathcal{L}_{i}^{(t+1)E+1/4}\right)\leq \lambda L_2^2\eta^2EG^2|\mathcal{D}_i|+\mu \ell_{kl}.
	\end{aligned}
\end{equation*}
\\Next, we analyze the GPC component. We begin by simplifying the knowledge distillation part,
\begin{equation*}
\begin{split}
    \sum_{k\in[C]}\ell_{kl}(p_{T_i,k}^t, p_{S_i,k}^t) &= \sum_{k\in[C]} p_{T_i,k}^t 
    \log \left\{ \frac{
        \displaystyle \exp \left( \frac{L t_{g,i,k}^t}{\tau} \right)
    }{
        \displaystyle \sum_{c \in [C]} \exp \left( \frac{L t_{g,i,c}^t}{\tau} \right)
    } 
    \middle/ 
    \frac{
        \displaystyle \exp \left( \frac{L t_{i,k}^t}{\tau} \right)
    }{
        \displaystyle \sum_{c \in [C]} \exp \left( \frac{L t_{i,c}^t}{\tau} \right)
    } 
    \right\}
    \\&\leq\sum_{k\in[C]}p_{T_i,k}^t\frac{Lt_{g,i,k}^t-Lt_{i,k}^t}\tau.
\end{split}
\end{equation*}
Then using Cauchy-Schwarz $\sum_{k=1}^nu_k v_k\leq\left(\sum_{k=1}^nu_k^2\right)^{1/2}\left(\sum_{k=1}^nv_k^2\right)^{1/2}$ and according to Assumption~\ref{aas4}, we have
\begin{equation*}
    \begin{split}
    \sum_{k\in[C]}p_{T,k}^t\cdot\frac{Lt_{g,i,k}^t-Lt_{i,k}^t}{\tau}&\leq\left(\sum_{k\in[C]}{p_{T_i,k}^t}^2\right)^{1/2}\left[\sum_{k\in[C]}\left(\frac{Lt_{g,i,k}^t-Lt_{i,k}^t}\tau\right)^2\right]^{1/2}\\&\leq\left[\sum_{k\in[C]}\left(\frac{Lt_{g,i,k}^{t}-Lt_{i,k}^{t}}{\tau}\right)^{2}\right]^{1/2}\\&\leq\frac{1}{\tau}\left\|g(\phi^{t}_{g,i})-g_{i}(\phi_{i}^{t})\right\|_{2}.
    \end{split}
\end{equation*}
Thus, we can bound the regularization term related to the classifier as follows:
\begin{equation*}
    \mathcal{L}_i^{(t+1)E+1/2}-\mathcal{L}_i^{(t+1)E+1/4}\leq \lambda L_2^2\eta^2EG^2|\mathcal{D}_i|+\frac{\mu}{\tau} \left( \left\| g(\phi^{(t+1)E+1/2}_{g,i}) - g_i\left(\phi_i^{(t+1)E+1/2}\right) \right\|_2 - \left\| g(\phi^{(t+1)E}_{g,i}) - g_i\left(\phi_i^{(t+1)E}\right) \right\|_2 \right).
\end{equation*}
Since the local classifier is not updated using the global classifier, we have $g_i\left(\phi_i^{(t+1)E+1/2}\right) = g_i\left(\phi_i^{(t+1)E}\right)$. By applying the inequality $\left\|A - C\right\|_2 - \left\|B - C\right\|_2 \leq \left\|A - B\right\|_2$, we obtain
\begin{equation*}
    \mathcal{L}_i^{(t+1)E+1/2}-\mathcal{L}_i^{(t+1)E+1/4}\leq \lambda L_2^2\eta^2EG^2|\mathcal{D}_i|+\frac\mu \tau\left\|g(\phi^{(t+1)E+1/2}_{g,i})-g(\phi^{(t+1)E}_{g,i})\right\|_2.
\end{equation*}
\noindent\textbf{{Remark.}} 
{The above inequality shows that the effect of the distillation regularization is explicitly controlled by the $1/\tau$ factor and remains bounded. 
This ensures that the additional distillation term can be consistently incorporated into the convergence analysis without altering the overall descent property or theoretical framework.}

According to Assumption~\ref{aas4} while introducing the zero term (here, the combination of $\theta_{i}^t$ and $\phi_{g,i}^{t}$ is treated as the local model), and following a similar analysis as in Lemma~\ref{ale2} (taking the expectation over $\gamma$ for the GPC component), we obtain the following expression for the global information update process:
\begin{equation*}
\begin{split}
    \mathcal{L}_i^{(t+1)E+1/2}-\mathcal{L}_i^{(t+1)E+1/4}&\leq \lambda L_2^2\eta^2EG^2|\mathcal{D}_i|+\frac{L_2\mu}\tau\left\|\sum_{i=1}^N\iota_i^{tE}\left(\omega_i^{(t+1)E}-\omega_i^{tE}\right)+\eta\sum_{e=1/2}^{E-1} gr^{tE+e}_i\right\|_2\\&
    \leq \lambda L_2^2\eta^2EG^2|\mathcal{D}_i|+\frac{L_2\mu\eta}\tau\left\|-\sum_{i=1}^N\iota_i^{tE}\sum_{e=1/2}^{E-1}gr_i^{tE+e} +\sum_{e=1/2}^{E-1}gr^{tE+e}_i\right\|_2,
\end{split}
\end{equation*}
\begin{equation*}
    \mathbb{E}\left(\mathcal{L}_{i}^{(t+1)E+1/2}\right) - \mathbb{E}\left(\mathcal{L}_{i}^{(t+1)E+1/4}\right)\leq \lambda L_2^2\eta^2EG^2|\mathcal{D}_i|+\frac{2\mu \eta L_2 E G}\tau.
\end{equation*}
\end{proof}

Now, we establish our final theoretical result.
\begin{theorem}\label{athe}
Suppose that Assumptions~\ref{aas1}-\ref{aas4} hold. The variation of the loss function between two adjacent communication rounds is bounded by
\begin{equation}\label{aeq1}
    \begin{aligned}\mathbb{E}\left(\mathcal{L}_i^{(t+1)E+1/2}\right)-\mathcal{L}_i^{tE+1/2}\leq&\eta\left[\sum_{e=1/2}^{E-1}\left(L_1\eta\left\|\nabla\mathcal{L}_i^{tE+e}\right\|_2^2-\left\|\nabla\mathcal{L}_i^{tE+e}\right\|_2^2+\nabla\mathcal{L}_i^{tE+e}-\sum_{i=1}^N\iota_i^{tE}\nabla\mathcal{L}_i^{tE+e}\right)\right]+\\&\eta\left(L_1\eta E\sigma^2+3\lambda\eta L_2^2EG^2|\mathcal{D}_i|+\frac{2\mu L_2EG}\tau\right).
    \end{aligned}
\end{equation}
If the following conditions are satisfied:
\begin{equation*}
    \eta_{e^{\prime}}<\frac{\sum_{e=1/2}^{e^{\prime}}\left(\left\|\nabla\mathcal{L}_i^{tE+e}\right\|_2^2+\sum_{i=1}^N\iota_i^{tE}\nabla\mathcal{L}_i^{tE+e}-\nabla\mathcal{L}_i^{tE+e}\right)-\frac{2\mu L_{2}EG}\tau}{\sum_{e=1/2}^{e^{\prime}}L_1\left\|\nabla\mathcal{L}_i^{tE+e}\right\|_2^2+L_1E\sigma^2+3\lambda L_2^2EG^2|\mathcal{D}_i|},
\end{equation*}
\begin{equation*}
    \mu_{e^{\prime}}<\frac{\tau\left[\sum_{e=1/2}^{e^{\prime}}\left(\left\|\nabla\mathcal{L}_i^{tE+e}\right\|_2^2+\sum_{i=1}^N\iota_i^{tE}\nabla\mathcal{L}_i^{tE+e}-\nabla\mathcal{L}_i^{tE+e}\right)\right]}{2L_{2}EG},
\end{equation*}
where $e^{\prime}=1/2, 1, ..., E-1$, we have that the right-hand side of Eq.~\ref{aeq1} is less than zero. This ensures a decrease in the loss function with each round, ultimately leading to convergence.
\end{theorem}

\clearpage
\section{Aggregation Weight Difference}\label{agg}
We present the changes in aggregation weights for all clients across different global aggregation rounds, starting from the second round. 
Specifically, we compute the difference between the aggregation weight of the current round and that of the previous round. 
The results, as shown in Figure \ref{aggre}, indicate that since the client datasets remain unchanged throughout the training process, the competitiveness of local models (measured by comparing their performance) exhibits minimal variation. Consequently, the aggregation weight updates are extremely small (note that the x-axis unit in the KDE plot is \textbf{1e-5}).
\begin{figure*}[!htbp]
	\centering
	\includegraphics[width=7in]{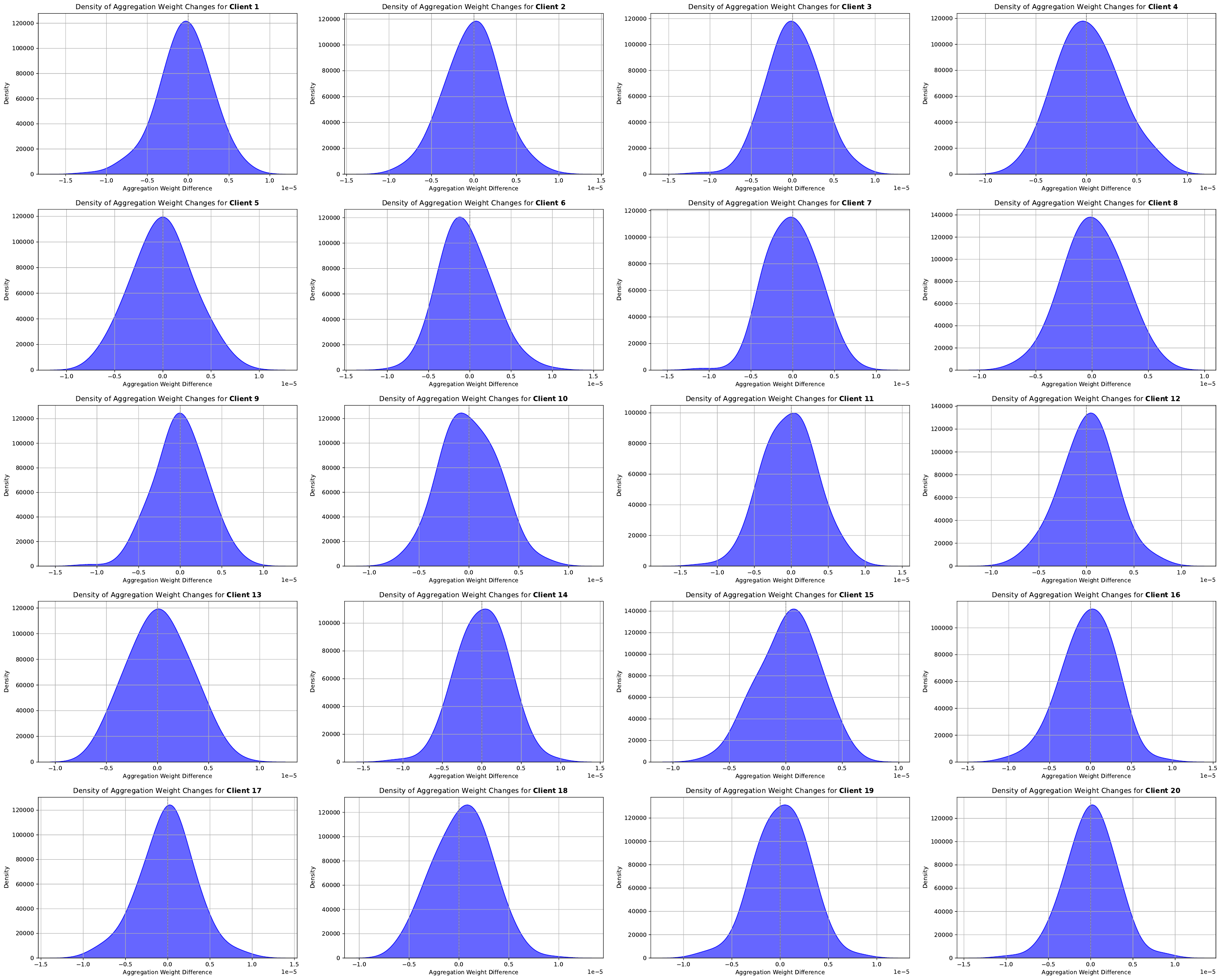}
	\caption{A Kernel Density Estimation (KDE) plot is used to illustrate the distribution of aggregation weight changes. In this analysis, a Gaussian kernel function is applied, with 100 rounds of global communication, and Scott’s rule is used to determine the bandwidth. 
    The resulting distribution is concentrated around \textbf{0}, indicating that the weight changes are minimal.}
	\label{aggre}
\end{figure*}

\clearpage
\section{Supplementary Notes and Experiments}

\subsection{{Notation Summary}}
{For clarity and consistency, Tab.~\ref{tab:table1} summarizes the key mathematical notations used in this paper, serving as a convenient reference for interpreting the theoretical formulations and methodological discussions.}
\begin{table}[H]
    \caption{Summary of Notations Used in the Paper\label{tab:table1}}
    \centering
    \renewcommand{\arraystretch}{1.2}
    \setlength{\tabcolsep}{4pt}
    \begin{tabular}{>{\centering\arraybackslash}p{1.6cm} >{\arraybackslash}p{6.2cm}}
        \toprule
        \textbf{Notation} & \textbf{Explanation} \\
        \midrule
        $N$ & Number of clients \\
        $\mathcal{D}$ & Entire dataset \\
        $i$ & Client index \\
        $\gamma$ & Mini-batch data \\
        $C$ & Category set \\
        $\mathcal{X}$ & Input space \\
        $\mathcal{H}$ & Feature space \\
        $\omega$ & Model parameter \\
        $\theta$ & Feature extractor parameter \\
        $\phi$ & Classifier parameter \\
        $f$ & Feature extraction function \\
        $g$ & Classification function \\
        $t$ & Current training round \\
        $An$ & Feature anchor \\
        $Lt$ & Logits output \\
        $\mathbb{L}_2$ & Euclidean norm \\
        $\eta$ & Learning rate \\
        $\mu$ & Hyperparameter for the GFA component \\
        $\lambda$ & Hyperparameter for the GPC component \\
        $fro$ & Frozen parameters (unchanged during training) \\
        $\mathcal{L}$ & Loss function \\
        $\ell_{ce}$ & Cross-entropy loss \\
        $\ell_{kl}$ & Kullback-Leibler divergence loss \\
        $\Omega$ & Center loss of feature representation \\
        \bottomrule
    \end{tabular}
\end{table}

\subsection{{Supplementary Experiments}}
{To further evaluate convergence and stability, we conduct experiments under a heterogeneous setting with $s=20$ on CIFAR-100. As shown in Fig.~\ref{fig9}, the method converges around the 80th iteration while maintaining stable accuracy.}
\begin{figure}[H]
    \centering
    \includegraphics[width=0.4\columnwidth]{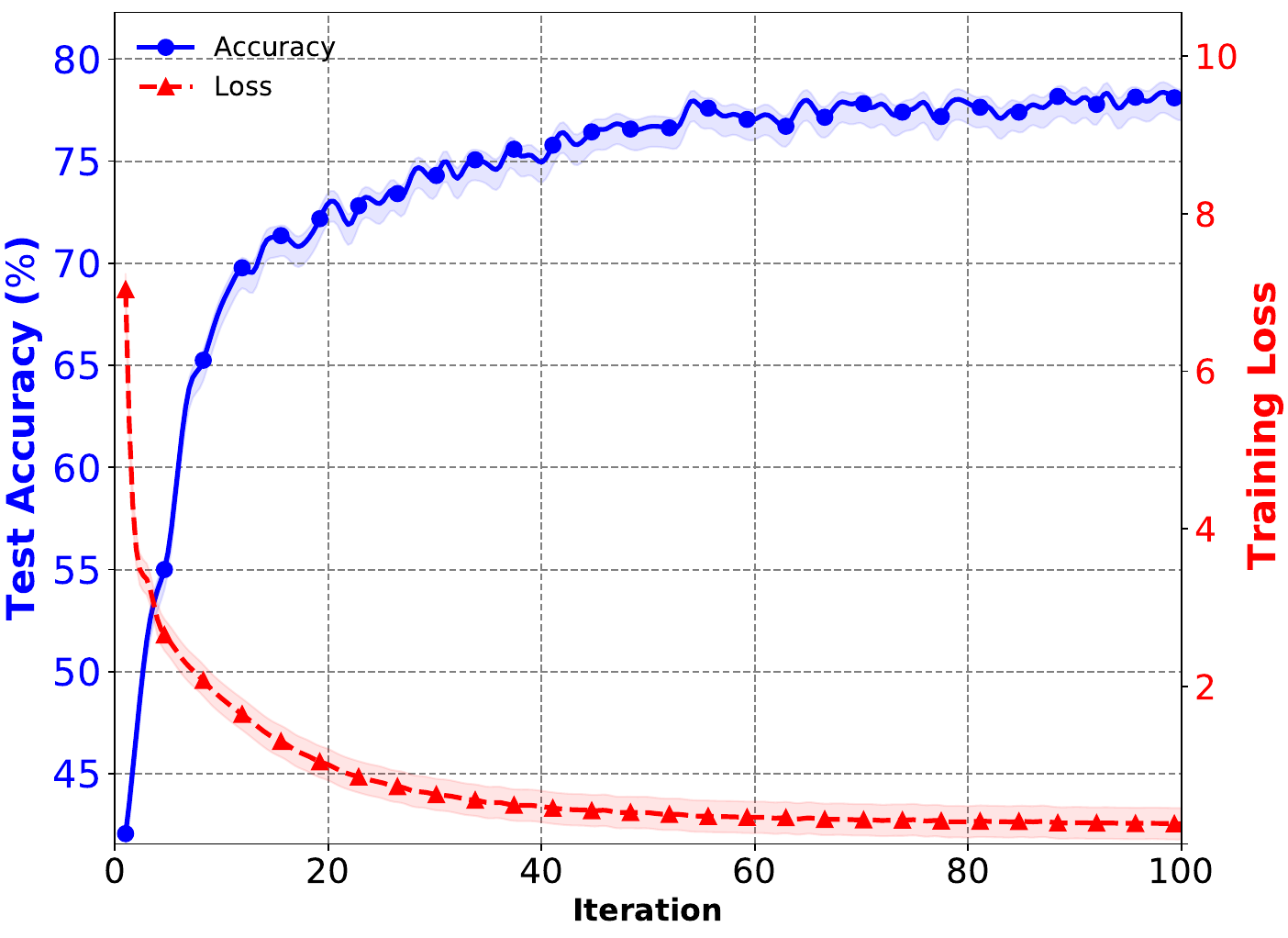}
    \caption{Training loss and test accuracy of FedeCouple on CIFAR-100 with $s=20$, showing convergence near the 80th iteration with stable accuracy.}
    \label{fig9}
\end{figure}